\documentclass[11pt]{article} 
\usepackage{fullpage}
\usepackage{natbib}
\usepackage{float}
\usepackage[colorlinks=true, allcolors=blue]{hyperref}
\usepackage{tabularx,colortbl,booktabs}
\usepackage{enumitem,color}

\usepackage{hyperref}
\usepackage[hyphenbreaks]{breakurl}

\usepackage{algorithm}
\usepackage{stackengine}
\usepackage[export]{adjustbox}
\usepackage{booktabs,enumitem}
\usepackage{makecell}
\usepackage{array}
\usepackage{newfloat}
\usepackage{url}            
\usepackage{booktabs}       
\usepackage{amsfonts}       
\usepackage{nicefrac}       
\usepackage{microtype}      
\usepackage[dvipsnames]{xcolor}         
\usepackage{siunitx}
\usepackage{subcaption,enumitem}
\usepackage{makecell}
\usepackage[export]{adjustbox}
\usepackage{graphicx}
\usepackage{footnote}
\usepackage{listings}
\definecolor{codegreen}{rgb}{0,0.6,0}
\definecolor{codegray}{rgb}{0.5,0.5,0.5}
\definecolor{codepurple}{rgb}{0.58,0,0.82}

\lstdefinestyle{mystyle}{
    backgroundcolor=\color{white},
    commentstyle=\color{codegreen},
    keywordstyle=\color{magenta},
    numberstyle=\tiny\color{codegray},
    stringstyle=\color{codepurple},
    basicstyle=\ttfamily\footnotesize,
    breakatwhitespace=false,
    breaklines=true,
    captionpos=b,
    keepspaces=true,
    numbers=left,
    numbersep=5pt,
    showspaces=false,
    showstringspaces=false,
    showtabs=false,
    tabsize=2
}

\makesavenoteenv{tabular}
\makesavenoteenv{table}
\usepackage{amsmath,amssymb,amsthm,amsfonts,bm,bbm, nicefrac}
\usepackage{algorithm}
\usepackage{algpseudocode}
\usepackage{mathtools}

\DeclareMathAlphabet{\mathsfit}{\encodingdefault}{\sfdefault}{m}{sl}
\SetMathAlphabet{\mathsfit}{bold}{\encodingdefault}{\sfdefault}{bx}{n}

\def\gA{{\mathcal{A}}}
\def\gB{{\mathcal{B}}}

\def\gD{{\mathcal{D}}}

\def\gH{{\mathcal{H}}}

\def\gL{{\mathcal{L}}}

\def\gS{{\mathcal{S}}}
\def\gT{{\mathcal{T}}}

\def\gX{{\mathcal{X}}}

\def\sB{{\mathbb{B}}}

\def\sN{{\mathbb{N}}}

\def\sR{{\mathbb{R}}}

\newcommand{\E}{\mathbb{E}}

\newcommand{\hattheta}{\widehat{\theta}}
\newcommand{\hatq}{\widehat{q}}
\newcommand{\hatw}{\widehat{w}}

\newcommand{\hatu}{\widehat{u}}

\newcommand{\hata}{\widehat{a}}
\newcommand{\utrunc}{u_\textup{trunc}}
\newcommand{\pipromptwise}{\pi^{\UCBabbv}}
\newcommand{\pipromptwiset}{\pi^{\UCBabbv}_t}
\newcommand{\piextend}{\pi^{\UCBabbv}_{\textup{extend}}}
\newcommand{\piextendt}{\pi^{\UCBabbv}_{\textup{extend}, t}}
\newcommand{\tildeu}{\widetilde{u}}
\newcommand{\tildea}{\tilde{a}}

\DeclareMathOperator*{\argmax}{arg\,max}
\DeclareMathOperator*{\argmin}{arg\,min}

\newcounter{protocol}
\makeatletter
\newenvironment{protocol}[1][htb]{%
  \let\c@algorithm\c@protocol
  \renewcommand{\ALG@name}{Protocol}
  \begin{algorithm}[#1]%
  }{\end{algorithm}
}
\makeatother

\newenvironment{enumerate*}%
{\begin{enumerate}[leftmargin=*,topsep=0pt]%
		\setlength{\itemsep}{0pt}%
		\setlength{\parskip}{0pt}}%
	{\end{enumerate}}

\newcommand{\UCBabbv}{\textup{PromptWise}}
\newcommand{\CB}{\textup{Cost-Aware Contextual MAB}}
\newcommand{\CBabbv}{\textup{CA-CMAB}}

\theoremstyle{plain}
\newtheorem{theorem}{Theorem}
\newtheorem{proposition}{Proposition}
\newtheorem{lemma}{Lemma}

\theoremstyle{definition}
\newtheorem{definition}{Definition}
\newtheorem{assumption}{Assumption}
\theoremstyle{remark}
\newtheorem{remark}{Remark}

\def\showcomments{1}
\ifnum\showcomments=1
\newcommand{\xy}[1]{[\textcolor{violet}{}]}

\definecolor{mplblue}{RGB}{31, 119, 180}
\definecolor{darkorange}{RGB}{255, 140, 0}

\title{PromptWise: Online Learning for Cost-Aware Prompt Assignment in Generative Models}

\author{
   Xiaoyan Hu\thanks{
	Department of Computer Science and Engineering, The Chinese University of Hong Kong. Email Address: \texttt{\{xyhu21,pick,farnia\}@cse.cuhk.edu.hk}}
   \and
   Lauren Pick\footnotemark[1]
   \and
   Ho-fung Leung\thanks{Independent Researcher. Email Address: \texttt{ho-fung.leung@outlook.com}}\footnotemark[2]
   \and
   Farzan Farnia\footnotemark[1]
}

\date{}

\begin{document}

\maketitle

\begin{abstract}
The rapid advancement of generative AI has provided users with a wide range of well-trained models to address diverse prompts. When selecting a model for a given prompt, users should weigh not only its performance but also its service cost. However, existing model-selection methods typically emphasize performance while overlooking cost differences. In this paper, we introduce PromptWise, an online learning framework that assigns prompts to generative models in a cost-aware manner. PromptWise estimates prompt–model compatibility to select the least expensive model expected to deliver satisfactory outputs. Unlike standard contextual bandits that make a one-shot decision per prompt, PromptWise employs a cost-aware bandit structure that allows sequential model assignments per prompt to reduce total service cost. Through numerical experiments on tasks such as code generation and translation, we demonstrate that PromptWise can achieve performance comparable to baseline selection methods while incurring substantially lower costs. The code is available at: \url{github.com/yannxiaoyanhu/PromptWise}.

\end{abstract}

\allowdisplaybreaks

\section{Introduction}
\label{sec:intro}
Recent advances in large language models (LLMs) and text-guided generative AI have expanded access to a broad array of services~\citep{Rombach_2022_CVPR, touvron2023llamaopenefficientfoundation, openai2024gpt4technicalreport, deepseekai2025deepseekv3technicalreport, geminiteam2025geminifamilyhighlycapable}. Users can now select among many well-trained models to address diverse prompts. This abundance raises a fundamental question: How should one assign prompts to models in a principled and accurate  manner?

A straightforward strategy is to choose the model with the highest averaged performance score and apply it universally to all prompts. However, the recent literature highlights that performance may vary significantly across prompt categories. A model may excel on one class of prompts yet perform suboptimally on others~\citep{qin2024diffusiongpt, frick2025prompt, hu2025onlinelearningapproachpromptbased}. Such discrepancies arise from differences in training data and architectural choices. To exploit this heterogeneity, prior studies propose either offline selector models~\citep{qin2024diffusiongpt, frick2025prompt} or online bandit algorithms~\citep{hu2025onlinelearningapproachpromptbased} that adaptively allocate prompts to specialized models.

\begin{figure*}[!t]
    \centering
    \includegraphics[width=0.85\linewidth]{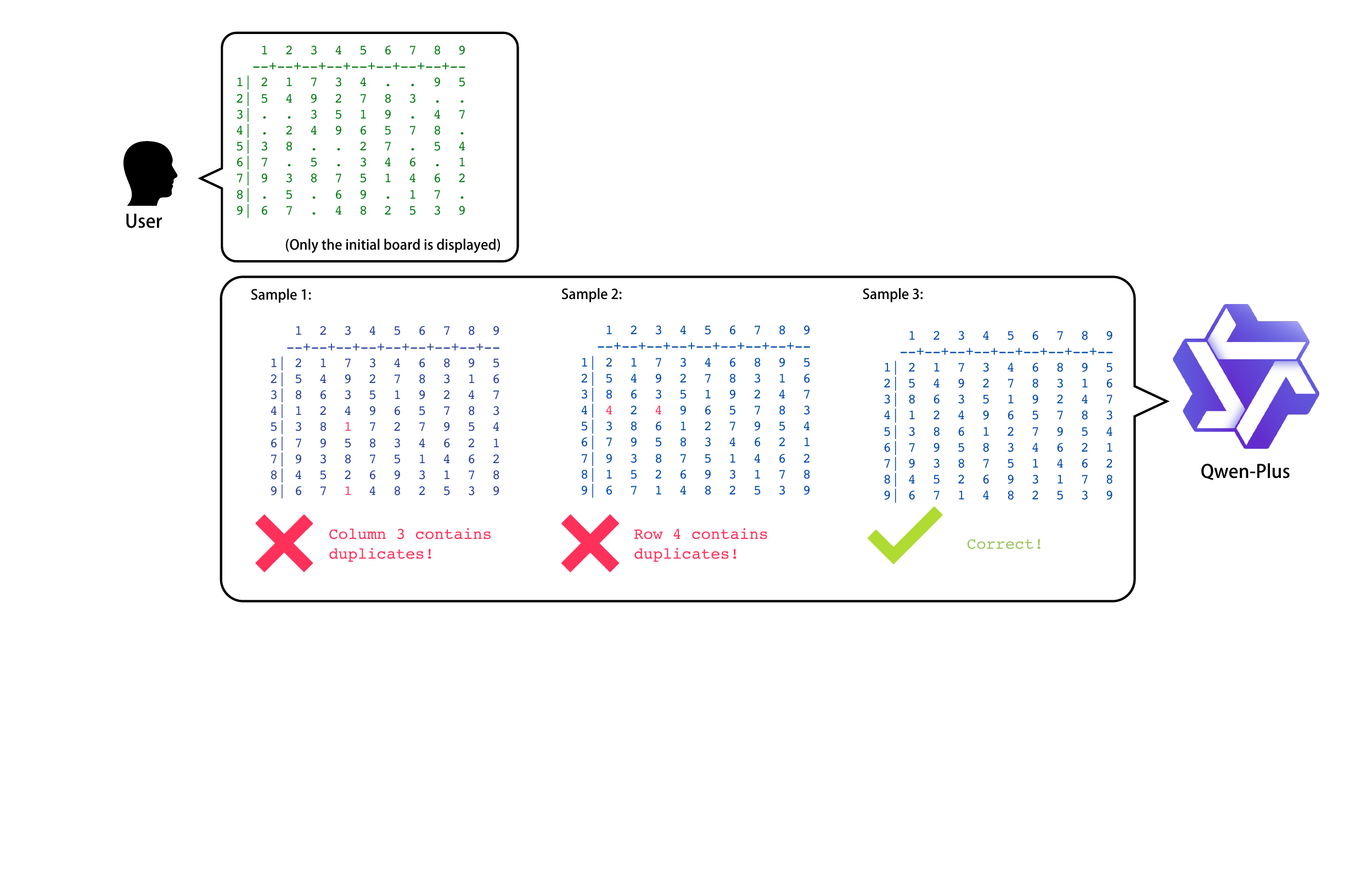}
    \caption{Illustration of cost-aware prompt assignment in Sudoku solving, where a single prompt may be routed to multiple models. PromptWise, via online learning, would route easier puzzles to the inexpensive \texttt{GPT-4o-mini}, while for harder instances it may escalate to stronger and more expensive available models \texttt{GPT-4o} and \texttt{o1}.}
    \label{fig1}
\end{figure*}

While these approaches improve prompt-specific performance, they all assume each prompt is \emph{assigned to exactly one model}, overlooking the role of monetary and computational costs in users' decisions on how to assign prompts to available models. In practice, users are not constrained to single-shot decisions. They may query multiple models for the same prompt, escalating to more expensive options only if cheaper attempts fail, as long as the cumulative expense is justified. 

For instance, consider the Sudoku example in Figure~\ref{fig1}. Suppose a user can choose among OpenAI’s \texttt{o1} (\$60 per million tokens (PMT)), \texttt{GPT-4o} (\$10 PMT), and \texttt{GPT-4o-mini} (\$0.6 PMT)\footnote{\url{https://platform.openai.com/docs/pricing}}. A seemingly cost-efficient strategy would be to first attempt the least expensive \texttt{GPT-4o-mini}, escalate to \texttt{GPT-4o} if unsuccessful, and resort to \texttt{o1} only for the hardest puzzles. Although the cheapest model may fail on complex tasks with high probability, the 100 times price difference could justify probing it first. Such multi-assignment strategies are natural in real-world deployments, yet absent from existing online model-selection frameworks.

In this work, we introduce \emph{PromptWise}, a framework for cost-aware online prompt-to-model assignment. At its core, PromptWise is a cost-aware contextual multi-armed bandit (CA-CMAB) formulation that explicitly incorporates both pricing heterogeneity and the possibility of multiple assignments per prompt. Unlike standard contextual bandits~\citep{NIPS2007_4b04a686, 10.1145/1772690.1772758, pmlr-v15-chu11a}, which commit to a single decision per round, CA-CMAB tracks both the best achieved performance per prompt and the cumulative cost of all attempts (Figure~\ref{fig:ca-cmab}).  

Specifically, \UCBabbv \, extends the classical Upper Confidence Bound (UCB) framework~\citep{892116, MAL-024} with a new sum-min regret formulation. For each prompt, only the best outcome across attempted models contributes to the performance regret, while costs are fully accumulated. This structure formalizes the trade-off between predictive reward and monetary expenditure, bridging the gap between theoretical online learning formulations and the realities of commercial generative AI services.  

We evaluate PromptWise through both controlled simulations, where ground-truth optimal strategies are known, and real-world generative AI tasks, including Sudoku and chess puzzles as well as code generation and translation benchmarks (HumanEval~\citep{chen2021evaluatinglargelanguagemodels}, HumanEval-X~\citep{zheng2023codegeex}, BigCodeBench~\citep{zhuo2024bigcodebench}). Across settings, PromptWise rapidly adapts to prompt distributions, achieves accuracy comparable to cost-unaware baselines, and reduces expenditures by large margins. In summary, this work's contributions are threefold:
\begin{itemize}[leftmargin=*]
    \item We formalize the cost-aware multi-assignment prompt-selection problem and highlight its practical importance for generative AI services.
    \item We propose \UCBabbv, a CA-CMAB algorithm that balances performance and cost via a novel sum-min regret metric.
    \item We demonstrate through simulations and real-world benchmarks that PromptWise achieves substantial cost reductions without sacrificing accuracy.
\end{itemize}

\begin{figure*}[!t]
    \centering
    \includegraphics[width=0.85\linewidth]{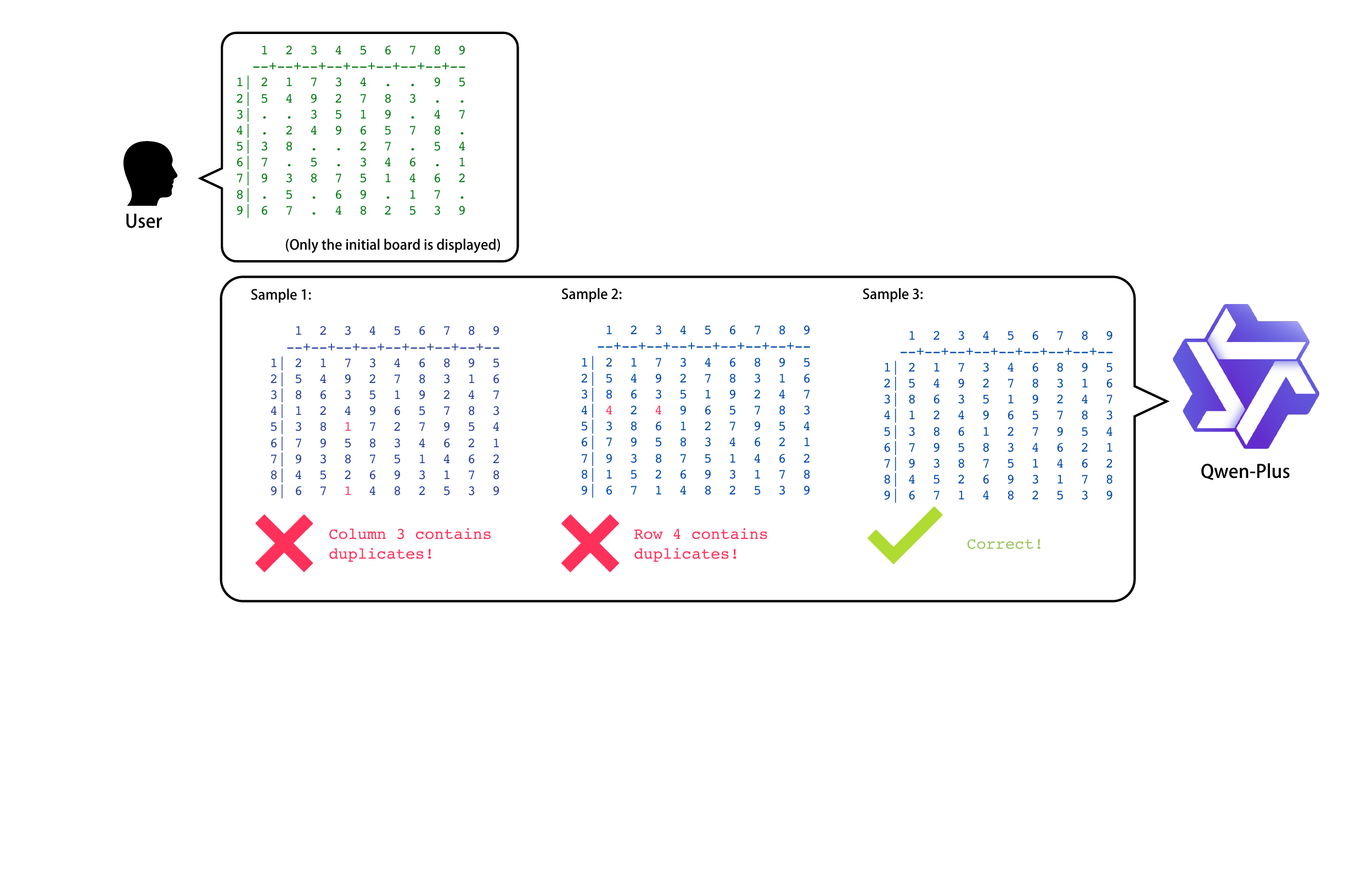}
    \caption{Interaction protocol of standard contextual bandits (top) versus our cost-aware contextual multi-armed bandit (CA-CMAB) framework in PromptWise (bottom). Unlike the standard setting, CA-CMAB enables multiple model assignments per prompt and explicitly balances accuracy (reward) with cumulative service cost, better reflecting practical usage scenarios.}
    \label{fig:ca-cmab}
\end{figure*}

\section{Related Work}
\label{sec:related_work}
\noindent\textbf{Generalized linear contextual bandits.}\quad Contextual bandit (CB) is a widely used framework for sequential decision-making with side information, which has been extensively used in recommendation systems and advertising. A relevant assumption in the CB setup is that the expected reward conditioned to an incoming context admits a generalized linear model~(GLM)~\citep{NIPS2010_c2626d85, pmlr-v70-li17c}. Particularly, the logistic bandit framework~\citep{pmlr-v119-faury20a, pmlr-v99-dong19a} has demonstrated stronger performance in applications where rewards are binary.

\noindent\textbf{Bandit problems with cost constraints.}\quad A line of study considers multi-armed bandits~(MAB) in a cost-aware setup, where each arm is associated with a cost distribution. The learner aims to not only optimize the cumulative return but also minimize the cost of pulling the arms. \citet{Ding_Qin_Zhang_Liu_2013} study cost-aware bandits with a constraint on the total paying cost. In Multi-fidelity MAB~\citep{NIPS2016_2ba59664, wang2023multifidelity}, the learner may choose to observe (pseudo-)rewards of an arm at some fidelity level and incur the corresponding cost. Recently, \citet{kanarios2024cost} study cost-aware best-arm identification, where the goal is to identify the largest reward arm using the minimum expected cost. To the best of our knowledge, cost-aware bandit algorithms in the literature consider optimizing a linear utility, e.g., the cumulative return minus (a regularized) cumulative cost.

\noindent\textbf{Online learning methods for training and selecting generative models.}\quad A line of research incorporates online learning and reinforcement learning in the training and selecting generative models. \citet{grnarova2017onlinelearningapproachgenerative} formulate the problem of training a Generative Adversarial Network~(GAN) as finding a mixed strategy in a zero-sum game. They propose an online learning-based training algorithm, which provably converges to an equilibrium. \citet{daskalakis2018traininggansoptimism} show that optimistic mirror descent ~(OMD) can address the limit cycling issue in training GANs. More recent work proposes the reinforcement learning from human feedback~(RLHF) pipeline for training large language models~(LLMs)~\citep{NEURIPS2022_b1efde53, NEURIPS2023_a85b405e, pmlr-v235-munos24a, dong2024rlhfworkflowrewardmodeling} and image-based generative models~\citep{NEURIPS2023_be93b165, NEURIPS2023_33646ef0}. Regarding the selection of generative models, recent work proposes online learning-based algorithms~\citep{hu2025onlinelearningapproachpromptbased} to a optimize performance score among a group of generative models. \citet{wang2025mixllmdynamicroutingmixed} considers finding a routing policy among a group of LLMs to maximize a linear utility of the response quality and the model cost.

\section{{\CB}}
\label{sec:pre}

\textbf{Interaction protocol.}\quad In this section, we introduce the framework of {\CB}~({\CBabbv}). The interaction protocol is given in Protocol~\ref{ca-cmab}. At each \textit{step} $t \in [T]$, where $[T] := \{1, 2, \cdots, T\}$, a \textit{context} $x_t \in \gX \subseteq \sB^d := \{ x \in \sR^d: \| x \|_2 \le 1 \} $~(e.g., the input text prompt) is drawn from the contextual distribution $p_0$. Different from the standard contextual bandit setting~\citep{NIPS2007_4b04a686}, the learner can interact under the context $x_t$ for multiple \emph{rounds} by taking a sequence of actions in $\Bar{\gA} := \{ a_0 \} \cup \gA$, where $a_0$ is the \textit{null action}, i.e., the user decides to terminate the process for prompt $x_t$ and moves to the next prompt, and $\gA$ is the set of \textit{arms}. 

Specifically, after pulling an arm $a_{t,i} \in \gA$ at the $i$-th round, the learner observes a reward $r_{t,i} \sim R_{a_{t,i}}(x_t)$ and pays a fixed cost $c_{t,i} := c_{a_{t,i}}$. The environment moves to the next task if the learner chooses the null action $a_0$ or the allocated round budget is completely used up~(i.e., $i = \tau_{\max}$). We denote by $\tau_t \in [\tau_{\max}] := \{ 0, 1, \cdots, \tau_{\max}\}$ the (random) number of arms pulled within step $t$. For convenience, we define $a_{t,\tau_t + 1} =: a_0$. Note that we assume that taking the null action yields a zero reward and cost, i.e., $R_{a_0} = c_{a_0} = 0$, which means that the action of moving to the next prompt costs nothing and also does not provide any reward. In this paper, we exclusively focus on the scenario where the rewards are \textit{binary}; however, the bandit framework can further extend to settings with non-binary rewards. 

\begin{assumption}[Binary reward]
\label{aspt:reward}
For each arm $a \in \gA$, the binary reward distribution follows $\text{Ber}(q_a(x))$ given the context outcome $x \in \gX$, where the success probability $q_a(x) \in [0,1]$.
\end{assumption}

\begin{remark}
The assumption of binary reward has been widely considered in the bandit literature, including Bernoulli bandits~\citep{02852583-a86c-3394-91f4-10c713c6e48c, pmlr-v23-agrawal12} and logistic bandits~\citep{pmlr-v99-dong19a, pmlr-v119-faury20a}. We note that a binary reward suits applications with success/failure outputs, such as puzzle-solving~(correct or incorrect solution) and code generation~(passing all the test cases or not). 
\end{remark}

\begin{protocol}[!t]
\begin{algorithmic}[1]
\caption{{\CB} ({\CBabbv})}
\Require step $T$, context distribution $p_0$, set of actions $\Bar{\gA} = \{a_0\} \cup \gA$, cost $\{ c_a \}_{a \in \Bar{\gA}}$, reward distributions $\{ R_a \}_{a \in \Bar{\gA}}$, cost parameter $\lambda > 0$, round budget $\tau_{\max} \in \sN_+$.
\For{step $t=1,2,\cdots,T$}
    \State The environment draws context $x_t \sim p_0$.
    \State Follow action $a_{t,1}$ and spend budget $c_{a_{t,1}}$.
    \State Initialize the counter $\tau_t \leftarrow 0$.
    \While{$a_0$ not selected \textbf{and} $\tau_t \le \tau_{\max}$}
        \State Select action $a_{t,\tau_t+1}$ and incur cost $c_{a_{t,\tau_t+1}}$.
        \State The environment draws reward $r_{t,\tau_t} \sim R_{a_{t,\tau_t}}(x_t)$.
        \State Increment the counter $\tau_t \leftarrow \tau_t + 1$
    \EndWhile
\EndFor
\label{ca-cmab}
\end{algorithmic}
\end{protocol}

\noindent \textbf{Policy and cost-aware utility.}\quad The learner interacts with the environment for $T$ steps. Upon receiving context $x_t \in \gX$ at step $t$, the learner plays (a sequence of) actions according to some policy $\pi_t$, which may depend on the previous observations. The learner aims to find a policy that maximizes the following \textit{utility} function, which trades off between the maximum obtained reward and the cumulative paid cost, where $\lambda \ge 0$ is called the \textit{cost coefficient} parameter:
\begin{equation}
\label{obj}
\begin{aligned}
    \pi^\star(x_t) := \argmax_{\pi} u^\pi(x_t), \textup{ where } u^\pi(x_t) & := \E_\pi \left[ \max_{ i = 1,\cdots,\tau_t } r_{t,i} - \lambda \sum_{i=1}^{\tau_t}c_{t,i} \right]
\end{aligned}
\end{equation}

\begin{remark}[Comparison to contextual bandits]
Figure~\ref{fig:ca-cmab} highlights the key difference between our proposed CA-CMAB framework and the standard contextual bandit~(CB) setting. When applying a standard CB algorithm to the prompt assignment task, the algorithm is supposed to have only one model selection per prompt. On the other hand, CA-CMAB allows the learner to act under a context multiple times before moving to the next task (i.e., prompt). Therefore, the CA-CMAB framework allows a cost-aware strategy to attempt the less expensive options even if they have a marginal chance of success, since a likely failure outcome can be followed by the selection of more expensive models to address the same prompt. However, the one-round interaction of standard CB for each prompt limits the bandit algorithm to attempt models with higher success chance, since the learner is supposed to move to next prompt in the subsequent round.
\end{remark}

Under the assumption of binary rewards and an unlimited budget for each round, we characterize the optimal policy in the following proposition. The proof is in Appendix~\ref{sec: a.1}.

\begin{proposition}[Oracle]
\label{thm-oracle}

Under Assumption~\ref{aspt:reward} and with unlimited round budget, an optimal policy~(\ref{obj}) is as follows: for any incoming context $x \in \gX$, it takes action $a^\star(x)$ if a reward of \num{1} is not observed, which is given by
\begin{equation}
\label{opt-act}    a^\star(x) =   \begin{cases}
        a_0, \;\quad\textup{if } \max_{a \in \gA}\{ q_a(x) - \lambda c_a\} \le 0, \\
        \argmin_{a \in \gA} \frac{c_a}{q_a(x)}, \quad\quad\quad\;\,  \textup{otherwise},
    \end{cases}
\end{equation}
where $q_a(x)$ and $c_a$ are the success probability conditioned to context $x$ and the (fixed) cost for any arm $a \in \gA$, respectively. The expected total cost of this policy is $\frac{c_{a^\star(x)}}{q_{a^\star(x)}}$, and the expected utility is given by
\begin{equation}
\label{utility}
    u^{\star}(x) = \mathbbm{1}\left[ a^\star(x) \neq a_0 \right] \cdot \left( 1 - \lambda \cdot \frac{c_{a^\star (x)}}{q_{a^\star (x)}(x)} \right).
\end{equation}
\end{proposition}

\section{The {\UCBabbv} Algorithm}
\label{sec:alg}
In this section, we propose {\UCBabbv} as an online learning algorithm to address the prompt-based model selection task in the CA-CMAB setting. Given the input prompt playing the role of context variable, {\UCBabbv} selects the action according to \ref{opt-act} according to the estimated success probabilities for every arm, i.e., the prompt-based generative model. As the success probabilities are estimated using the limited previous observations and could be different from the true success probability (based on the true underlying distribution of the models), directly plugging them into Equation~(\ref{opt-act}) can lead to a sub-optimal choice of generative models. To address the issue in the online selection scenario, our proposed {\UCBabbv} applies the Upper Confidence Bound~(UCB) approach~\citep{892116, MAL-024} widely used in the bandit literature, to avoid sub-optimal model assignment policies in the online selection process.

\begin{figure*}[!t]
    \centering
    \includegraphics[width=0.32\linewidth]{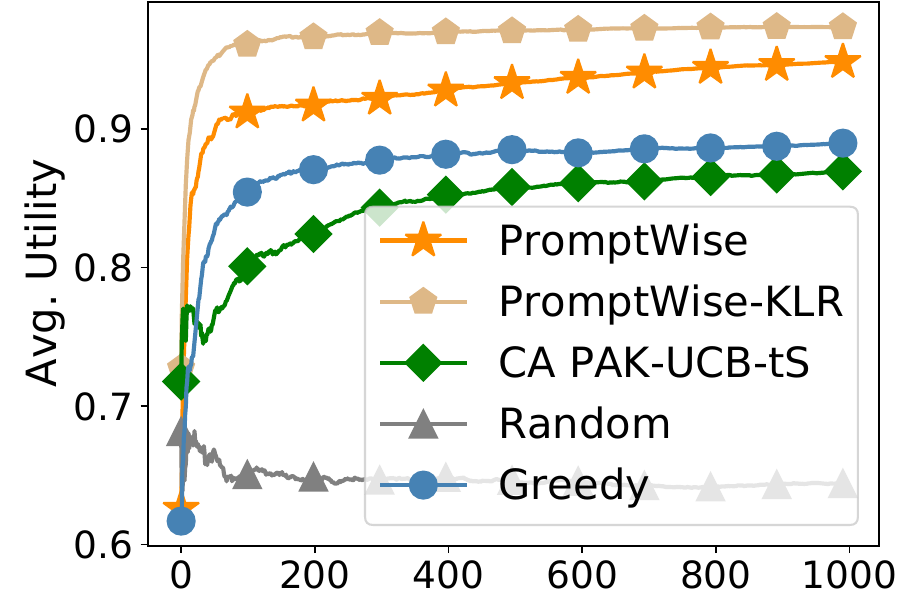}
    \hfill
    \includegraphics[width=0.32\linewidth]{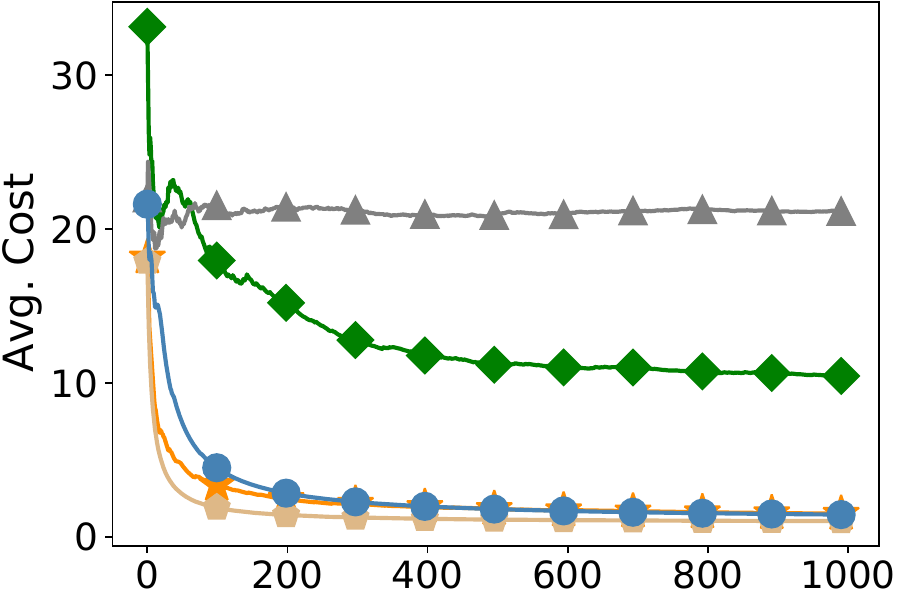}
    \hfill
    \includegraphics[width=0.32\linewidth]{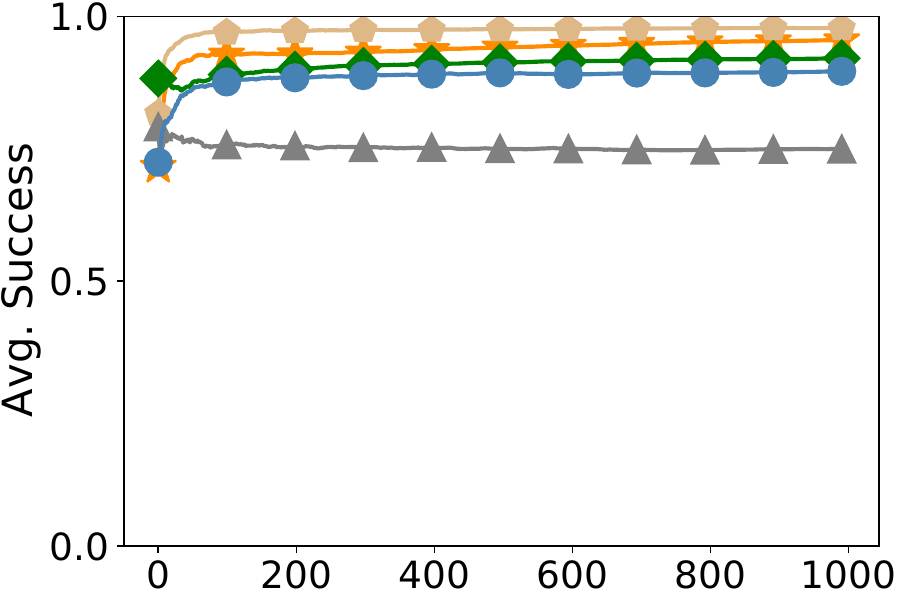}
    \caption{Code Completion on HumanEval~(Task 2): Gemini-2.5-Flash-preview, Deepseek-Chat, Qwen-Plus, GPT-4o, and Claude-Opus-4. Results are averaged over \num{20} trials.}
    \label{fig:humaneval}
\end{figure*}

\noindent \textbf{Overview of {\UCBabbv}.}\quad We now introduce {\UCBabbv}~(presented in Algorithm~\ref{alg:ucb}). In the exploration phase~(Lines 1-4), {\UCBabbv} plays every arm for $\tau_{\exp}$ initial rounds, where $\tau_{\exp} \in \sN_+$ is a tuning parameter. Upon observing the reward values, {\UCBabbv} calculates the maximum-likelihood estimation (MLE) of the weight vectors $\{ \theta_a^\star: a \in \gA \}$, which is given by
\begin{align}
\label{mle-1}
    \hattheta_a := \underset{\theta \in \sR^d}{\arg\!\max} \Bigg\{ \widehat{\ell}_a(\theta):= \sum_{(X,R) \in \gD_a} \Bigl[ R\cdot\log \mu\bigl(\langle \theta, X \rangle \bigr) + (1 - R)\cdot\log \left(1 - \mu\bigl(\langle \theta, X \rangle\bigr)\right) \Bigr] \Bigg\}.
\end{align}
The MLE is utilized to estimate the success probabilities~(Lines 3 and 6). At any subsequent step~(Lines 5-17), {\UCBabbv} interacts with the environment by attempting to compute the oracle action~(\ref{opt-act}) on a set of estimated success probabilities~(Lines 7-15). We include an exploration bonus when computing the estimate $\hatq_a$, which is given by
\begin{align}
\label{opt-hatq}
    \hatq_a &:= \mu\left(x_t^\top \hattheta_a + \alpha \|x_t\|_{V_a^{-1}} \right), \textup{ where } V_a = \sum_{(X,R) \in \gD_a} XX^\top.
\end{align}
The design of the exploration bonus follows the UCB methods for the generalized linear contextual bandits~\citep{NIPS2010_c2626d85, pmlr-v70-li17c}, and it can be shown that $\hatq_a$ upper bounds the groundtruth value $q_a(x_t)$ with a high probability. To interact with the environment, {\UCBabbv} attempts to predict the optimal action~(\ref{opt-act}) by replacing the groundtruth success probabilities with their UCB estimates~(\ref{opt-hatq}). Specifically, {\UCBabbv} takes the null action $a_0$ and moves on to the next step if all the arms are predicted to fail to address the task~(Line 9). Otherwise, {\UCBabbv} plays the arm that yields the highest predicted utility~(Line 11). After receiving a new observation, {\UCBabbv} performs inner-step updates to both the MLE and the estimated success probability for the played arm~(Line 14). Hence, {\UCBabbv} can utilize the knowledge of the observations in the earlier rounds to adjust the arm selection in the subsequent rounds.

\begin{algorithm}[!t]
\begin{algorithmic}[1]
\caption{\UCBabbv}
\Require step $T$, cost parameter $\lambda > 0$, regression dataset $\gD_a \leftarrow \varnothing$ for $a \in \gA$, $V_a \leftarrow \bm{0}_{d \times d}$, round budget $\tau_{\max}$, tuning parameter $\tau_\textup{exp}$ and $\alpha$
\For{each arm $a \in \gA$}
    \State Collect samples $\gD_a \leftarrow \gD_a \cup \{ (x_s, r_s) \}_{s = 1}^{\tau_{\exp}}$ by pulling arm $a$ (and moving to the next context) upon observing contexts $\{x_s\}_{s = 1}^{\tau_{\exp}}$, which are drawn from the environment.
    \State Compute the MLE $\widehat{\theta}_a$ defined in Equation~(\ref{mle-1}) by gradient ascent.
\EndFor
\For{step $t=|\gA|\tau_\textup{exp} + 1, \cdots, T$}
    \State Observe context $x_t$ and compute the estimated success probability $\hatq_a$ for each arm $a \in \gA$, which is given by Equation~(\ref{opt-hatq}).
    \While{not playing null action $a_0$}
        \If{$\max_{a \in \gA} \{\hatq_a - \lambda\cdot c_a \} \le 0$ \textbf{or} hitting the round budget \textbf{or} receiving a reward of 1}
        \State Play null action $\hata \leftarrow a_0$.
        \Else
        \State Play arm $\hata \leftarrow \argmin_{a \in \gA} \frac{c_a}{\hatq_a}$.
        \State The environment draws a reward $r$.
        \State Update dataset $\gD_{\hata} \leftarrow \gD_{\hata} \cup \{ (x_t, r) \}$.
        \State Compute the MLE $\hattheta_a$ and estimate $\hatq_{\hata}$.
        \EndIf
    \EndWhile
\EndFor
\label{alg:ucb}
\end{algorithmic}
\end{algorithm}

\textbf{Extension to kernel methods.}\quad The original {\UCBabbv} algorithm utilizes a linear function (inside the logistic function) to predict the success probability. To further improve the predictions, one approach is to incorporate non-linearity using kernel methods. In the following, we propose {\UCBabbv}-KLR~(see Algorithm~\ref{alg:ucb-klr}) as an extension of {\UCBabbv} to kernel-based prediction functions. Let $k: \sR^d \times \sR^d \to \sR$ denote a kernel function, e.g., the radial basis function~(RBF) kernel $k_\text{RBF}^\sigma(x, x') = \exp( - \frac{ \| y - y' \|_2^2 }{2\sigma^2}  ) $. Given the regression dataset $\gD_a$, we denote by $\bm{K}_{\gD_a} := [k(x,x')]_{(x, r), (x',r') \in \gD_a} \in \sR^{|\gD_a| \times |\gD_a|}$ the Gram matrix, and the corresponding (squared) matrix norm is given by $\| w \|_{\bm{K}_{\gD_a}}^2 := w^\top \bm{K}_{\gD_a} w$ for every $w \in \sR^{|\gD_a|}$. Let $\beta \ge 0$ denote a tuning parameter. {\UCBabbv}-KLR leverages kernel logistic regression~(KLR)~\citep{NIPS2001_2eace51d, wahba2018soft} to predict the success probability for the incoming context $x_t$:
\begin{equation*}
    \hatq_a(x_t) = \mu\left( \sum_{(X,R) \in \gD_a} \hatw_a^X k(X, x_t) + \alpha \cdot \gB ( x_t | \gD_a) \right),
\end{equation*}
where $\gB ( x_t | \gD_a)$ is the bonus term and the weight $\hatw_a := [\hatw_a^X]_{(X, R) \in \gD_a}^\top \in \sR^{|\gD_a|}$ is solved from minimizing the (regularized) negative log-likelihood:
\begin{equation*}
\begin{aligned}
    \hatw_a &:= \argmin_{w \in \sR^{|\gD_a|}} \ell (w),
\end{aligned}
\end{equation*}
where we consider the objective function:
\begin{equation*}
\begin{aligned}
    \ell (w) = - \sum_{(X,R) \in \gD_a} \Big[R \cdot \log \hatq_a(X, w|\gD_a) + (1 - R) \cdot \log (1 - \hatq_a(X, w|\gD_a)) \Big] + \beta \| w \|^2_{ \bm{K}_{\gD_a} }
\end{aligned}
\end{equation*}
Similar to the {\UCBabbv} algorithm, the bonus term is designed to quantify the uncertainty in predicting the success probabilities for the incoming context $x_t$. One possible choice is given by $\gB( x_t | \gD_a ) = \beta^{ -\frac{1}{2} } ( k(x_t, x_t) - k_{x_t}^\top ( \bm{K}_{\gD_a} + \beta I)^{-1} k_{x_t} )^{\frac{1}{2}}$ in the kernelized contextual bandit literature~\citep{valko2013finitetimeanalysiskernelisedcontextual}, where $k_{x_t} = [k(x_t, X)]_{(X,R) \in \gD_a} \in \sR^{|\gD_a|}$.

\section{Numerical Experiment}
\label{sec:exp}

\begin{figure*}[!t]
    \centering
    \includegraphics[width=0.32\linewidth]{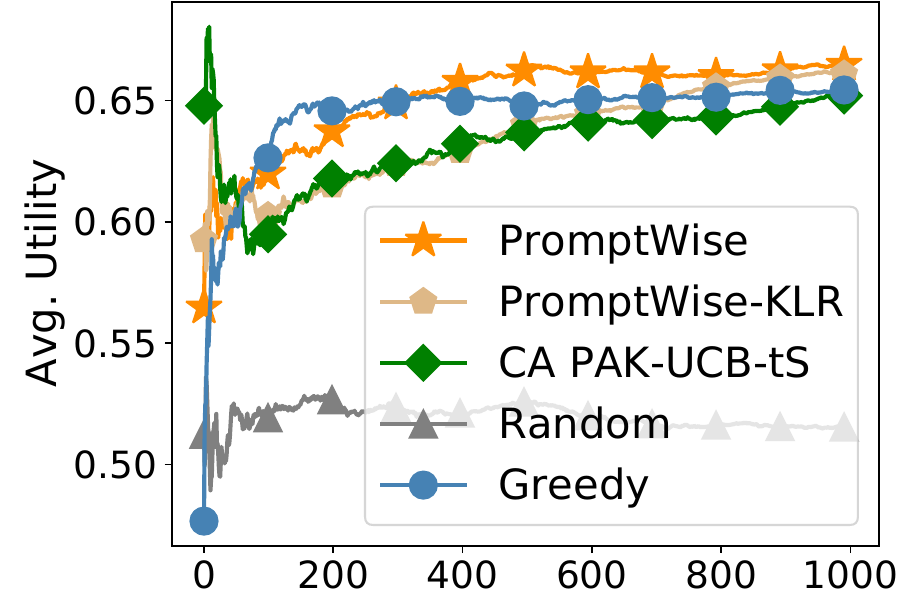}
    \hfill
    \includegraphics[width=0.32\linewidth]{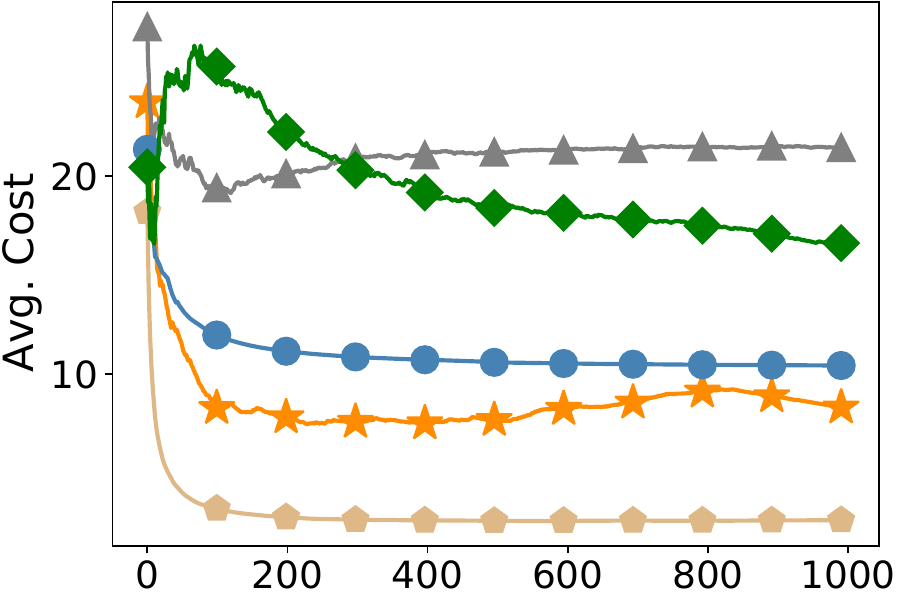}
    \hfill
    \includegraphics[width=0.32\linewidth]{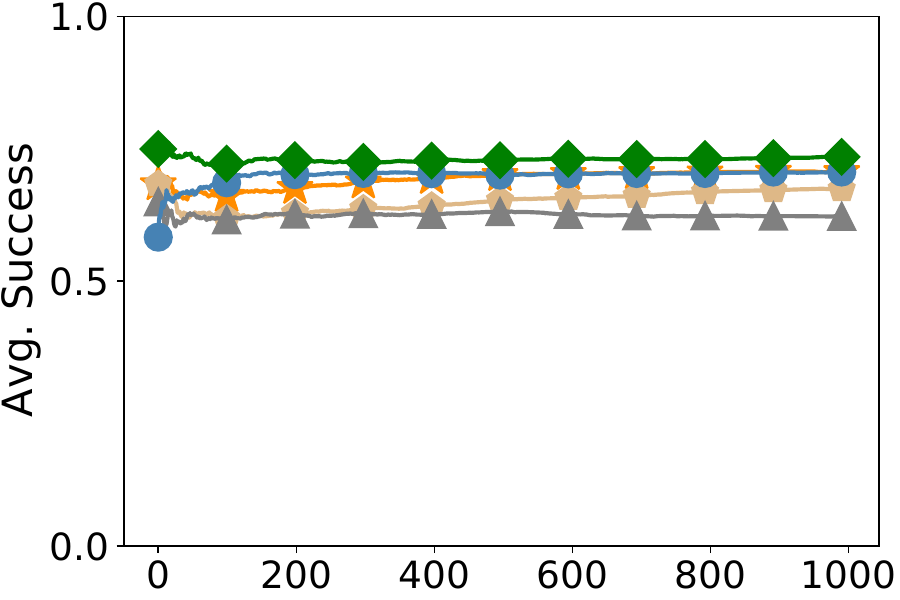}
    \caption{Code Translation on HumanEval-X Benchmark~(Task 3): Gemini-2.5-Flash-preview, Deepseek-Chat, Qwen-Plus, GPT-4o, and Claude-Opus-4. The LLM is provided with C++~(or Java) code and is asked to translate it into Java~(or C++). Results are averaged over 20 trials.}
    \label{fig:code-trans}
\end{figure*}

We test the proposed {\UCBabbv} algorithm and the baseline methods on various tasks, including solving random Sudoku and chess puzzles, code completion, code translation, and text-to-image generation. In the experiments, we report numerical results of the selection among a group of large language models~(LLM). Details and additional results can be found in Appendix~\ref{appendix:exp-details}.

\noindent \textbf{Summary.}\quad We highlight main findings in our numerical experiments. First, the proposed {\UCBabbv} algorithm can effectively balance between cost and model performance conditioned to the input prompt. Specifically, {\UCBabbv} strategically queries answers from lower-cost models first, proceeding to more expensive options only if the cheaper models cannot adequately address the prompt. Second, compared to the baseline methods, {\UCBabbv} can usually attain higher success rates (in addressing the input prompt) and incur lower cost, which showcases its superiority in cost-performance optimization. Third, {\UCBabbv} can adapt to new models and prompts in the selection process, which demonstrates its potential and practical efficacy in real-world applications.

\subsection{Experimental Setup}

\textbf{Baselines.}\quad We compare the proposed \textcolor{darkorange}{{\UCBabbv}}~(Algorithm~\ref{alg:ucb}) and \textcolor{Tan}{{\UCBabbv}-KLR}~(Algorithm~\ref{alg:ucb-klr}) algorithms to the following baseline methods: 1) \textcolor{ForestGreen}{CA PAK-UCB-tS}: the cost-aware variant of the contextual bandit (a single assignment per prompt) PAK-UCB algorithm~\citep{hu2025onlinelearningapproachpromptbased} using an RBF kernel, which seeks to maximize the following cost-aware objective, i.e., $\arg\max_{a \in \mathcal{A}} \{ \widehat{q}_a - \lambda c_a \}$. Specifically, CA PAK-UCB-tS is a kernel-based contextual bandit algorithm that selects a model according to the input prompt to maximize the expected reward at the beginning of each step and keeps sampling from the chosen model until a satisfactory answer is obtained or the budget is used up,
2) \textcolor{gray}{Random}: randomly selecting a model and then moving on to the next prompt, 
and 3) \textcolor{mplblue}{Greedy}: selecting the model that maximizes the (average) success rate and then moving on to the next prompt. In the Appendix, we report results for four additional baselines: 4) \textcolor{gray}{Random-till-Succeed~(RtS) ('- -')}: randomly selecting a model until a satisfactory answer is obtained or the round budget is used up, 5) \textcolor{mplblue}{Greedy-till-Succeed~(GtS) ('- -')}: selecting the model that maximizes the (average) success rate, until a satisfactory answer is obtained or the budget is used up, 6) \textcolor{BlueGreen}{Lowest-cost}: always selecting the cheapest model, and 7)~\textcolor{RedViolet}{Highest-cost}: always selecting the most expensive model. In the experiments, all the algorithms explore each arm at least once, except for the Lowest-cost and Highest-cost baselines.

\noindent\textbf{Performance metrics.}\quad We evaluate the algorithmic performance using the following metrics: 1) Avg. Utility: utility per step attained by the algorithm, i.e., $T^{-1} \sum_{t=1}^T u(x_t)$, 
2) Avg. Cost: cost per step paid by the algorithm, i.e., $T^{-1} \sum_{t=1}^T\sum_{i=1}^{\tau_t} c_{t,i}$, and 3) Avg. Success: the proportion of steps where the selected model successfully addresses the prompt.

\noindent\textbf{Candidate models and costs.}\quad We consider the following models in Tasks 1-3: Gemini-2.5-Flash-preview~(\$\num{0.6} PMT), Deepseek-Chat~(\$\num{1.1} PMT), Qwen-Plus~(\$\num{1.2} PMT), GPT-4o~(\$\num{10.0} PMT), and Claude-Opus-4~(\$\num{75.0} PMT). Our selection of candidate models encompasses a broad range of service providers, with significant price differences. In the experiments, the model costs are calculated as the total price of 1M input and output tokens~(See Table~\ref{tab:model-price} in the Appendix). 

\noindent\textbf{Hyperparameters.}\quad In all the experiments, the cost parameter is set to be $\lambda = 0.01$, and the round budget is set to be $\tau_{\max} = 5$. We set the hyperparameters in the {\UCBabbv} and {\UCBabbv}-KLR algorithms as $\tau_{\exp}=1$ and $\alpha=\sqrt{2\log(2|\mathcal{A}|/\delta)}$ following a standard choice in the UCB literature, where $
\delta=0.05$ is the failure probability. For the {\UCBabbv}-KLR algorithm, we utilize the RBF kernel with $\sigma = 3$, and we set the tuning parameter $\beta = 1$. 

\subsection{Empirical Results of PromptWise Applications}

\textbf{Results for Task 1 ``Puzzle-solving games''.}\quad In the first setup, we utilize LLMs to solve randomly generated $9\times9$ Sudoku games. The LLM is given the initial board with \num{5}-\num{30} blank entries. The sample prompt and answers are visualized in Figures~\ref{sudoku-prompt} and~\ref{sample-ans-sudoku} in the Appendix. For the {\UCBabbv}, {\UCBabbv}-KLR, and CA PAK-UCB-tS algorithms, we utilize a \num{5}-dimensional one-hot embedding as the input context, which is based on the number of blank entries in the initial board. Our results show that {\UCBabbv} can significantly outperform the baseline methods on utility and achieve a good balance between the success rate and the cost~(Figure~\ref{fig:sudoku-full} in the Appendix). 

In the second setup, we prompt LLMs to solve chess puzzles in the Lichess Puzzle Database\footnote{\url{https://database.lichess.org/\#puzzles}}. The LLM is given the initial chessboard in the FEN notation and is asked to provide the best single move. The sample prompt and answers are visualized in Figure~\ref{sample-ans-chess} in the Appendix. 
We utilize the RoBERTa embedding of the prompt as the input context to the {\UCBabbv}, {\UCBabbv}-KLR, and CA PAK-UCB-tS algorithms. Results are summarized in Figure~\ref{fig:chess-puzzle-full} in the Appendix.

\begin{table}[!h]
\centering
\begin{tabular}{c|c c}
    \toprule
    Algorithm & Success (\%) & Cost \\
    \midrule
     \textcolor{darkorange}{{\UCBabbv}} & \num{95.56}$\pm$\num{2.92} & 
     \num{1.49}$\pm$\num{0.72} \\
     \textcolor{Tan}{{\UCBabbv}-KLR} & $97.84\pm4.32$ & $1.01\pm0.02$ \\
     \textcolor{ForestGreen}{CA PAK-UCB-tS} & $92.16 \pm 9.92$ & $10.42 \pm 14.75$ \\
     \textcolor{gray}{Random} & $74.99\pm1.01$ & $21.16\pm1.03$ \\
     \textcolor{mplblue}{Greedy} & $89.67\pm2.97$ & $1.43\pm0.56$ \\
    \bottomrule
\end{tabular}
\caption{Performance of code completion on the HumanEval benchmark (Task 2): Results are averaged over 20 trials.}
\label{tab-human_eval}
\end{table}

\noindent\textbf{Results of Task 2 ``Code completion''.}\quad In these experiments, we select LLMs to complete Python code. We consider tasks from two benchmarks: HumanEval~\citep{chen2021evaluatinglargelanguagemodels} and BigCodeBench~\citep{zhuo2024bigcodebench}. Specifically, the model is provided with a function description~(written in natural language) and is asked to complete the code following the function declaration~(written in Python). Sample prompt and answers are visualized in Figures~\ref{code-gen-prompt} and~\ref{sample-code-completion} in the Appendix. We utilize the CodeBERT embedding as the input context to the {\UCBabbv}, {\UCBabbv}-KLR, and CA PAK-UCB-tS algorithms. More details can be found in Part 3 of Appendix~\ref{appendix:abl}. Our results show that {\UCBabbv} can consistently outperform the baseline methods on utility across the benchmarks~(Table~\ref{tab-human_eval}, Figures~\ref{fig:humaneval} and~\ref{fig:humaneval-full} for the HumanEval benchmark and Figure~\ref{fig:bigcodebench-full} for BigCodeBench).

To further demonstrate the trade-off between costs and success rates, we perform a set of experiments with different round budgets, where the algorithms select between two LLMs: Claude-Sonnet-4~(\$\num{18} PMT) and Claude-Opus-4 (\$\num{75} PMT). Results (Figure~\ref{fig:comparison}) show that our proposed {\UCBabbv} method yields a better trade-off between model costs and performance. 

\begin{figure}[!t]
    \centering
    \includegraphics[width=0.5\linewidth]{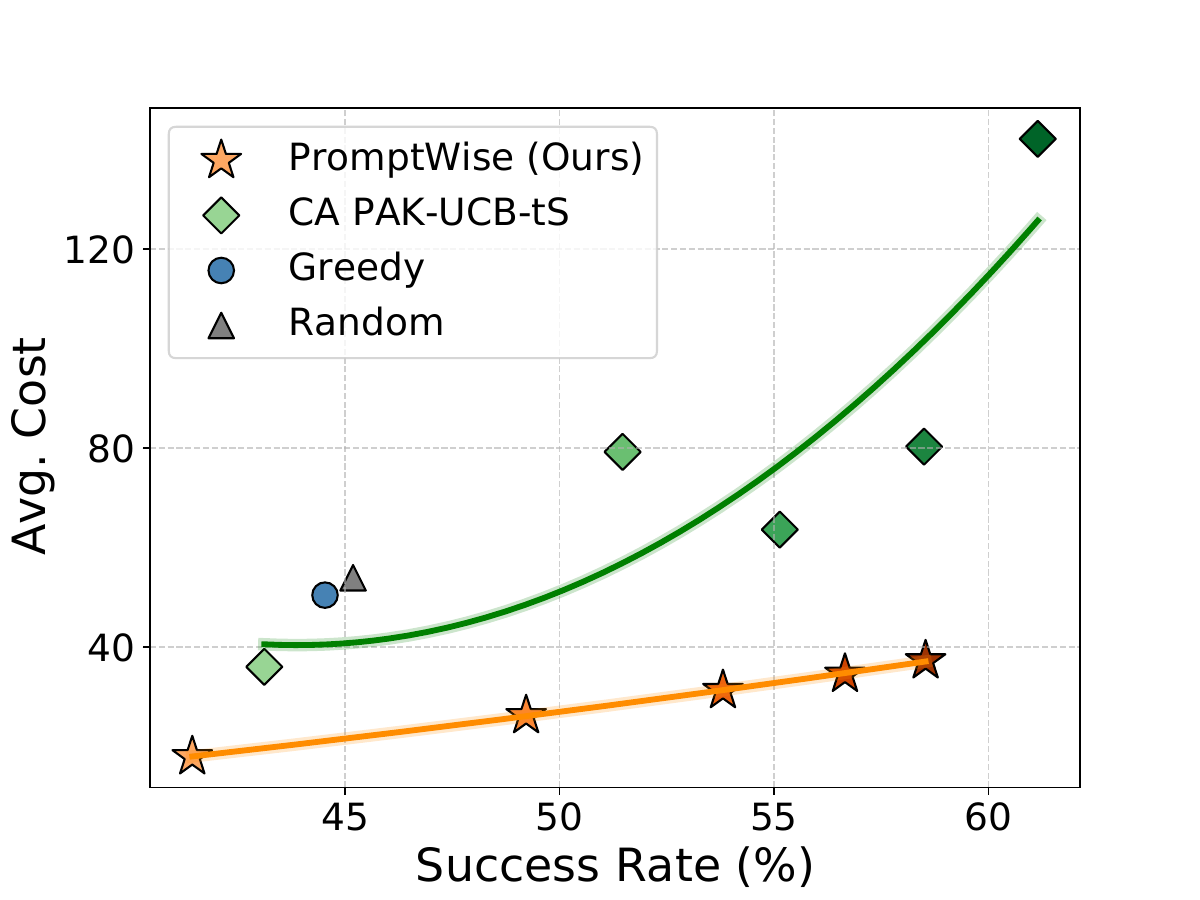}
    \caption{Trade-offs between cost and success rate of assigning prompts to models Claude-Sonnet-4 (\$18 PMT) and Claude-Opus-4 (\$75 PMT): We test the algorithms with different round budgets $\tau_{\max}=1,2,4,8,16$ (the darker color correspond to greater round budgets). The proposed {\UCBabbv} method could yield comparable success rates while incurring lower costs, achieving better trade-offs between model costs and performance. Note that for the Greedy and Random baselines, the averaged cost and success rate remain constant across different round budget values.}
    \label{fig:comparison}
\end{figure}

\noindent\textbf{Results for Task 3 ``Code translation''.}\quad In this task, we prompt LLMs to translate code between different programming languages. Specifically, the LLM is provided with a function written in the \textit{source} programming language~(e.g., C++) without any further description/explanation. Then, the model is asked to translate the code following the function declaration in the \textit{target} programming language~(e.g., Java). Sample prompt and answers are visualized in Figures~\ref{code-trans-prompt} and~\ref{sample-code-trans} in the Appendix. The tasks are randomly drawn from the HumanEval-X benchmark~\citep{zheng2023codegeex}. We utilize the CodeBERT embedding as the input context to the {\UCBabbv}, {\UCBabbv}-KLR, and CA PAK-UCB-tS algorithms. Results are summarized in Figures~\ref{fig:code-trans} and~\ref{fig:humaneval-x-full} in the Appendix.

\noindent\textbf{Results for Task 4 ``Text-to-image generation''.}\quad We test the performance of {\UCBabbv} on text-to-image (T2I) models on a synthetic setup, where prompts are randomly drawn from five categories in the MS-COCO dataset~\citep{lin2015microsoftcococommonobjects}. The five arms synthesized from Stable Diffusion v2~\citep{Rombach_2022_CVPR} are visualized in Figure~\ref{tab: Sync-T2I} in the Appendix, where the $i$-th arm can output high-quality images for the first $i$ prompt types. For each input prompt, the learner aims to generate an image \textit{up to some satisfactory level}~(i.e., the CLIPScore is higher than \num{30}). Results are in Figure~\ref{fig:sync-T2I-full} in the Appendix.

\section{Conclusion}
\label{sec:con}
This work addressed the problem of cost-aware online prompt assignment for generative models, where the objective is to optimize a utility function that balances generation performance with the cumulative cost of model queries. We proposed \UCBabbv, an online learning algorithm within the PromptWise framework, which enables multi-assignment of a single prompt to multiple models in order to achieve cost efficiency. The algorithm adaptively explores cheaper models first and escalates to more expensive alternatives only when indicated by the learning process.

Empirical evaluations on both simulated settings and real-world tasks, including puzzle solving, code generation, and image generation, showed that PromptWise achieves accuracy comparable to strong baselines while substantially reducing cost. These results demonstrate the value of combining multi-assignment with cost-aware exploration to improve the budget–performance trade-off.

A limitation of the current approach is its reliance on well-defined evaluation metrics. For domains such as code generation, where correctness can be assessed through test cases, this requirement is naturally satisfied. In contrast, for more general LLM applications without reliable automatic evaluation, the applicability of PromptWise depends on integration with user-driven or human-in-the-loop assessments.

Future research directions include extending the framework to other modalities including text-to-video generation, and investigating combinations of PromptWise with offline model selection or reward modeling methods. Such extensions could provide improved initialization and enhanced representations for cost-aware selection in practical deployments.

\bibliographystyle{abbrvnat}
\bibliography{ref}

\newpage
\appendix
\section{Proofs in Section~\ref{sec:pre}}

\subsection{Oracle with Unlimited Budget}
\label{sec: a.1}

\begin{proof}[Proof of Proposition \ref{thm-oracle}]
Fix a context $x$ and rewrite $q_a := q_a(x)$ for convenience. The problem can be formulated as an infinite-horizon Markov Decision Process~(MDP) with three states $\gS = \{0, 1, \gT\}$, where \num{0} and \num{1} represent $\max\{r_{t, 1}, \cdots, r_{t,i}\}$ up to the $i$-th round, and $\gT$ is the terminal state. The transition kernel is given by
\begin{equation*}
    P( s' | s, a) =
    \begin{cases}
        1, & \text{if } (s,a,s') \in \{(s,a_0,\gT)\}_{s \in \gS} \cup \{ (1, a, 1) \}_{a \in \gA} \\
        q_a, & \text{o.w.}
    \end{cases},
\end{equation*}
and the reward function is defined as $r(s,a,s')=u(s') - u(s) - \lambda c_a$, where $u(0) = u(\gT) = 0$ and $u(1) = 1$. The agent aims to maximize the cumulative reward $\sum_{i=1}^\infty r(s_i, a_i, s'_i)$. It is easy to verify that the optimal value function $V^\star(\gT) = V^\star(1) = 0$. At state \num{0}, it holds that $Q^\star(0,a) = q_a (1 - \lambda c_a) + (1 - q_a) (-\lambda c_a + V^\star(0))$ and $V^\star(0) = \max_{a \in \Bar{\gA}} Q^\star(0, a)$. Hence, we have
\begin{equation}
\label{eq-17}
    V^\star(0) = \max_{a \in \Bar{\gA}} \left\{ q_a - \lambda c_a + (1 - q_a) V^\star(0) \right\},
\end{equation}
and we can derive $V^\star(0) = \max\{ 0, \frac{q_{a^\star} - \lambda c_{a^\star}}{q_{a^\star}} \}$, where 
\begin{align*}
    a^\star = \argmax_{a \in \gA} \{ q_a - \lambda c_a + (1 - q_a) V^\star(0) \} = 
    \begin{cases}
        a_0 & \text{if } V^\star(0) = 0, \\
        \argmin_{a \in \gA} \frac{c_a}{q_a} & \text{o.w.}
    \end{cases}
\end{align*}
We conclude the proof.
\end{proof}

\section{Proofs in Section~\ref{sec:alg}}

\subsection{Sublinear Regret Bound for {\UCBabbv} Variant}
\label{sec:regret-var}

\textbf{Definition of regret.}\quad In this section, we provide theoretical justification for the proposed {\UCBabbv} algorithm. Specifically, we show that a {\UCBabbv} variant attains sublinear \textit{regret}, which is given by
\begin{equation}
\label{regret}
    \textup{Regret}(T) := \E \left[ \sum_{t = 1} ^T \left( u^\star(x_t) - u^{\pi_t}(x_t) \right) \right],
\end{equation}
where $\pi_t$ is the policy at step $t$, and we denote by $u^\star(x) := u^{\pi^\star(x)}(x)$ the optimal utility~(\ref{obj}) conditioned to any context $x \in \gX$. The regret metric is widely adopted in the online learning literature~\citep{MAL-024}. Note that any algorithm that attains sublinear regret guarantees (asymptotic) convergence to the optimal policy, i.e., $\frac{1}{T}\E \left[ \sum_{t = 1} ^T \left( u^\star(x_t) - u^{\pi_t}(x_t) \right) \right] \rightarrow 0$.

\noindent\textbf{{\UCBabbv} variant.}\quad Next, we present the variant algorithm in Algorithm~\ref{alg:ucb-var} below. Note that Equation~(\ref{mle-2}) is equivalent to solving Equation~(\ref{mle-1})~(The detailed derivation can be found in Appendix~\ref{p-mle}). The only difference is that Algorithm~\ref{alg:ucb-var} performs MLE updates only at the end of a step~(Line~12), instead of doing it each after observing a reward~(Lines 12-13 in Algorithm~\ref{alg:ucb}). As a result, Algorithm~\ref{alg:ucb-var} either plays the null action $a_0$ upon observing an input context and moves on to the next step~(Lines~7-8) or pulls an arm deterministically until succeeded or the round budget has been used up~(Lines~10).

\begin{algorithm}[!ht]
\begin{algorithmic}[1]
\caption{{\UCBabbv} with Per-Step Updates}
\Require step $T$, cost parameter $\lambda > 0$, regression dataset $\gD_a \leftarrow \varnothing$ for $a \in \gA$, $V_a \leftarrow \bm{0}_{d \times d}$, tuning parameter $\tau_\textup{exp}$ and $\alpha$
\For{$a \in \gA$}
    \State Collect $\gD_a \leftarrow \gD_a \cup \{ (x_s, r_s) \}_{s = 1}^{\tau_{\exp}}$ by pulling arm $a$ (and moving on to the next context) upon observing contexts $\{x_s\}_{s = 1}^{\tau_{\exp}}$ drawn from the environment.
    \State Calculate the maximum-likelihood estimate (MLE) $\hattheta_a$ by solving
    \begin{equation}
    \label{mle-2}
        \sum_{(X,R) \in \gD_a} (R - \mu(X^\top \theta)) X = 0. \tag{9}
    \end{equation}
\EndFor
\For{$t=|\gA|\tau_\textup{exp} + 1, \cdots, T$}
    \State Observe context $x_t$ and compute the estimated success probability for each arm $a \in \gA$:
    \begin{equation}
        \hatq_a \leftarrow \mu\left(x_t^\top \hattheta_a + \alpha \|x_t\|_{V_a^{-1}} \right), \textup{ where } V_a \leftarrow \sum_{(X,R) \in \gD_a} XX^\top. \tag{10}
    \end{equation}
    \If{$\max_{a \in \gA} \{\hatq_a - \lambda\cdot c_a \} \le 0$}
        \State Play null action $\hata \leftarrow a_0$ and move on to the next step.
    \Else
        \State Play arm $\hata \leftarrow \argmin_{a \in \gA} \frac{c_a}{\hatq_a}$ until receiving a reward of 1 or hitting the round budget.
        \State Update $\gD_{\hata} \leftarrow \gD_{\hata} \cup \{ (x_t, r_i) \}_{i=1}^{\tau_t}$ where $r_1,\cdots,r_{\tau_t}$ are the observed rewards.
        \State Recompute $\hattheta_{\hata}$ and $\hatq_{\hata}$ for the played arm by Equations~(9) and~(10), respectively.
    \EndIf
\EndFor
\label{alg:ucb-var}
\end{algorithmic}
\end{algorithm}

\noindent\textbf{Regret analysis.}\quad We state two assumptions in our analysis. The first one is a regularity assumption on the contextual distribution $p_0$, which is standard in the logistic bandit literature~\citep{pmlr-v70-li17c}

\begin{assumption}
There exists a constant $\sigma_0 > 0$ such that $\lambda_{\min}(\Sigma) \ge \sigma_0^2$, where $\Sigma := \E_{X \sim p_0}[XX^\top]$.
\end{assumption}

\begin{assumption}[Lower-bounded success rate]
\label{aspt-min-q}
There exists a constant $q_0$ such that $q_a(x) \ge q_0$ for any arm $a \in \gA$ and context $x$ in the support of distribution $p_0$.
\end{assumption}

Let $\kappa := \inf_{a \in \gA} \{ \mu^\prime(x^\top \theta_a^\star): \| x \| \le 1, \| \theta - \theta_a^\star \| \le 1 \}$. We derive the regret bound for Algorithm~\ref{alg:ucb-var} in the following theorem.

\begin{theorem}[Regret of Algorithm \ref{alg:ucb-var}]

Let the choice of $\tau_{\exp}$ follows from Lemma~\ref{lem-156} with $B = 1$ and a union bound over all the arms, i.e., $\tau_{\exp} = C \sigma_0^{-2} (d + \log(|\gA|/\delta))$ for some universal constant $C > 0$. Set $\alpha = \kappa^{-1} \sqrt{\frac{d}{2} \log(1 + 2 \tau_{\max} T/d) + \log(|\gA| / \delta) }$. Then, the regret of running Algorithm~\ref{alg:ucb-var} with round budget $\tau_{\max} \ge \frac{\log(d q_0 T^{-1/2})}{\log(1 - q_0)}$ for $T$ step is bounded by
\begin{equation*}
    \textup{Regret}(T) \le |\gA| \tau_{\exp} + \widetilde{O}\left((\kappa^{-1} + 1) d \sqrt{|\gA| T} \right),
\end{equation*}
where logarithmic factors are hidden in the notation $\widetilde{O}(\cdot)$.

\begin{proof}
\textbf{Optimism.}\quad Let $\hatq^t_a := \hatq^t(x_t)$ be the optimistic estimate~(\ref{opt-hatq}) for any arm $a \in \gA$ at (the beginning of) the $t$-th step. Note that $\hatq^t_a$ is identical across the rounds within step $t$. We first show that, with probability at least $1 - 2\delta$, it holds that $\hatq_a^t \ge q_a^t$ at any iteration $t \ge |\gA|\tau_{\exp} + 1$, where we denote by $q_a^t := q_a(x_t)$ the groundtruth expected value. By Lemma~\ref{lem-156} and the choice of $\tau_{\exp}$, we have $\lambda_{\min}(V_a) \ge 1$ for any $a \in \gA$ with probability at least $1 - \delta$. Let $\hattheta_a^t$ denote the maximum-likelihood estimate at the $t$-th iteration solving from Equation~(\ref{mle-2}). Note that
\begin{align}
\label{eq-80}
    \hatq_a^t = \mu \left(\langle x_t,\hattheta_a^t \rangle + \alpha \| x_t \|_{(V_a^t)^{-1}} \right) = \mu\left( \langle x_t,\theta_a^\star \rangle + \langle x_t, \hattheta_a^t - \theta_a^\star \rangle + \alpha \| x_t \|_{(V_a^t)^{-1}} \right).
\end{align}
By Lemma~\ref{lem-182} and a union bound over all the arms, with probability at least $1 - \delta$, we have
\begin{align*}
    \langle x_t, \hattheta_a^t - \theta_a^\star \rangle + \alpha \| x_t \|_{(V_a^t)^{-1}} \ge \| x_t \|_{(V_a^t)^{-1}} \left( -\| \hattheta_a^t - \theta_a^\star \|_{V_a^t} + \alpha \right) \ge 0,
\end{align*}
where the second inequality holds by the Cauchy-Schwarz inequality. Hence, it holds that $\hatq_a^t \ge q_a^t$ with probability at least $1 - 2 \delta$. Therefore, we have
\begin{equation}
\label{eq-92}
    \hatu_t := \mathbbm{1}[\hata_t \neq a_0] \cdot \left(1 - \argmin_{a \in \gA} \frac{\lambda c_a}{\hatq_a^t}\right) \ge u^\star(x_t),
\end{equation}
where we denote by $u^\star(x_t)$ the optimal utility given unlimited round budget, which is given in Equation~(\ref{utility}).

\noindent\textbf{Regret decomposition.}\quad To derive a upper bound for the regret, note that any step such that $\hata_t = a_0$ does not incur regret. Hence, we derive
\begin{align*}
    & \text{Regret}(T) \\
    = & \underbrace{\sum_{a \in \gA} \sum_{s = 1}^{ \tau_{\exp} } \left( u^\star(x_s) - (q_a(x_s) - \lambda c_a) \right)}_{\text{Exploration (Lines 1-8)}} + \underbrace{\sum_{t = |\gA| \tau_{\exp} + 1}^T \left( u^\star(x_t) - \utrunc^{\hata_t}(x_t) \right)}_\text{Lines 9-18} \\
    \le & |\gA|\tau_{\exp} + \sum_{t = |\gA| \tau_{\exp} + 1}^T \mathbbm{1}[\hata_t \neq a_0] \cdot \left( \hatu_t - u^{\hata_t}_\infty(x_t) + u^{\hata_t}_\infty(x_t) - \utrunc^{\hata_t}(x_t) \right) \tag{Inequality (\ref{eq-92})} \\
    \le & |\gA|\tau_{\exp} + \sum_{t = |\gA| \tau_{\exp} + 1}^T \mathbbm{1}[\hata_t \neq a_0] \cdot \left( \left( \frac{ \lambda c_{\hata_t} }{ q_{\hata_t}^t } - \frac{ \lambda c_{\hata_t} }{ \hatq_{\hata_t}^t } \right) + \left| \frac{ ( 1 - q_{\hata_t}^t )^{ \tau_{\max} } (q_{\hata_t}^t  - \lambda c_{\hata_t}) }{ q_{\hata_t}^t } \right| \right) \tag{Lemma~\ref{truncate-ut}} \\
    \le & |\gA|\tau_{\exp} + \sum_{t = |\gA| \tau_{\exp} + 1}^T \mathbbm{1}[\hata_t \neq a_0] \cdot \left( q_0^{-2} ( \hatq_{\hata_t}^t - q_{\hata_t}^t ) + \frac{ ( 1 - q_{\hata_t}^t )^{ \tau_{\max} } }{ q_{\hata_t}^t } \right) \tag{Assumption~\ref{aspt-min-q}},
\end{align*}
where in the first inequality, we denote by $\utrunc^{\hata_t}(x_t)$ the utility of the \emph{truncated} policy in Lines 7-13 of Algorithm~\ref{alg:ucb-var}, that is, playing the null action $a_0$ if $\hatu_t \le 0$, or otherwise, keep playing arm $\hata_t$ until receiving a reward of 1 or hitting the round budget. In addition, we define $ u^{\hata_t}_\infty(x_t) := 1 - \lambda c_{\hata_t} / q_{\hata_t} $ to be the utility with unlimited budget. Note that the gap $| u^{\hata_t}_\infty(x_t) - \utrunc^{\hata_t}(x_t) |$ can be directly derived from Lemma~\ref{truncate-ut}. In the last inequality, we utilize the fact that $\lambda c_{\hata_t} < \hatq^t_{\hata_t} \le 1$.

\noindent\textbf{Bound $ \hatq_{\hata_t}^t - q_{\hata_t}^t $.}\quad Note that $\mu^\prime(z) \le \frac{1}{4}$. From Equation~(\ref{eq-80}), we derive
\begin{align*}
    \hatq_a(x_t) - q_a(x_t) \le \frac{1}{4} \left( \langle x_t, \hattheta_a^t - \theta_a^\star \rangle + \alpha \| x_t \|_{(V_a^t)^{-1}} \right) \le \frac{\alpha}{2} \| x_t \|_{(V_a^t)^{-1}}.
\end{align*}
\noindent\textbf{Putting everything together.}\quad We have
\begin{align*}
    & \text{Regret}(T) \\
    < & |\gA|\tau_{\exp} + q_0^{-2}\sum_{t = |\gA| \tau_{\exp} + 1}^T \mathbbm{1}[\hata_t \neq a_0] \cdot ( \hatq_{\hata_t}^t - q_{\hata_t}^t ) + T (1 - q_0)^{\tau_{\max}} q_0^{-1} \\
    = & |\gA|\tau_{\exp} + T (1 - q_0)^{\tau_{\max}} q_0^{-1} + q_0^{-2} \alpha \sum_{t = |\gA| \tau_{\exp} + 1}^T \mathbbm{1}[\hata_t \neq a_0] \cdot \| x_t \|_{(V_a^t)^{-1}}.
\end{align*}
Let $n_a^{T+1}$ denote the size of the regression data set $\gD_a$ of any arm $a \in \gA$ after $T$ steps. Further, by Lemma~\ref{lem-182}, we have
\begin{align*}
    \sum_{t = |\gA| \tau_{\exp} + 1}^T \mathbbm{1}[\hata_t \neq a_0] \cdot \| x_t \|_{(V_a^t)^{-1}} \le & \sum_{a \in \gA} \sqrt{2 n_a^{T+1} d \log \left( \frac{ \tau_{\max}T }{d} \right)} \\
    \le & \sqrt{ 2|\gA|\tau_{\max} T d \log \left( \frac{ \tau_{\max}T }{d} \right) }.
\end{align*}
By the choice of $\tau_{\max} \ge \frac{\log(d q_0 T^{-1/2})}{\log(1 - q_0)}$, the total regret is bounded by
\begin{align*}
    \text{Regret}(T) \le |\gA|\tau_{\exp} + O(d\sqrt{T}) + \widetilde{O}(\kappa^{-1} d \sqrt{|\gA|T}),
\end{align*}
which concludes the proof.
\end{proof}

\end{theorem}

\subsection{Maximum-Likelihood Estimate}
\label{p-mle}

\begin{proof}
We show that the maximum-likelihood estimate~(\ref{mle-1}) can be derived by solving Equation~(\ref{mle-2}). Note that
\begin{align*}
    \widehat{\ell}_a(\theta) =& \sum_{(X,R) \in \gD_a} R \cdot \log \frac{\mu (\langle \theta, X \rangle)}{1 - \mu (\langle \theta, X \rangle)}  + \log (1 - \mu(\langle \theta, X \rangle)) \\
    = & \sum_{(X,R) \in \gD_a} R\cdot\langle \theta, X \rangle + \log (1 - \mu(\langle \theta, X \rangle)).
\end{align*}
By taking the derivative w.r.t $\theta$, we derive
\begin{align*}
    \widehat{\ell}^\prime_a(\theta) = \sum_{(X,R)} R\cdot X + \frac{-\mu^\prime(\langle \theta, X \rangle)}{1 - \mu(\langle \theta, X \rangle)} X = \sum_{(X,R)} (R - \mu(\langle \theta, X \rangle)) X,
\end{align*}
which concludes the proof.
\end{proof}

\subsection{Auxiliary Lemmas}

\begin{lemma}[{\citep[Proposition 1]{pmlr-v70-li17c}}]
\label{lem-156}

Define $V_n := \sum_{s=1}^n X_sX_s^\top$ where $X_s \in \sR^d$ is drawn i.i.d from distribution $p_0$ with support in the unit ball. Further, let $\Sigma = \E_{X \sim p_0} [XX^\top]$ be the second moment matrix, and $B$ and $\delta > 0$ be two positive constants. Then, there exist positive, universal constants $C_1$ and $C_2$ such that $\lambda_{\min}(V_n) \ge B$ with probability at least $1 - \delta$, as long as
\begin{equation*}
    |\gD_a| \ge \left( \frac{C_1\sqrt{d} + C_2\sqrt{\log(1/\delta)}}{\lambda_{\min}(\Sigma)} \right)^2 + \frac{2B}{\lambda_{\min}(\Sigma)}.
\end{equation*}
\end{lemma}

The following lemma is an adaptation of~\citep[Lemma 3]{pmlr-v70-li17c}.

\begin{lemma}
\label{lem-182}

Let $V_a^t$ be the empirical second moment matrix for a fixed arm $a \in \gA$ up to step $t$, and we denote by $n_a^t$ the size of the regression dataset $\gD_a$. At any iteration $t \ge |A|\tau_{\exp} $, we denote by $\hattheta_a^t$ the maximum-likelihood estimator corresponding to any arm $a \in \gA$. Suppose $\lambda_{\min}(V_a^{|\gA|\tau_{\exp} + 1}) \ge 1$. For any $\delta \in [T^{-1}, 1)$, it holds that
\begin{equation*}
    \| \hattheta_a^t - \theta_a^\star \|_{V_a^t} \le \kappa^{-1} \sqrt{\frac{d}{2} \log\left(1 + \frac{2 n_a^t}{d} \right) + \log\frac{1}{\delta}}
\end{equation*}
for all $t \ge |A|\tau_{\exp} $ with probability at least $1-\delta$.
\end{lemma}

\begin{lemma}[{\citep[Lemma 2]{pmlr-v70-li17c}}]
Let $\{ X_n \}_{n=1}^\infty$ be a sequence in $\sR^d$ satisfying $\| X_n \| \le 1$. Define $X_0 := \bm{0}$ and $V_n := \sum_{s = 0}^{n-1} X_sX_s^\top$. Suppose there is an integer $m$ such that $\lambda_{\min}(V_{m+1}) \ge 1$, then for all $s > 0$,
\begin{equation*}
    \sum_{n = m + 1}^{m + s} \| X_n \|_{V_n^{-1}} \le \sqrt{2nd \log\left(\frac{s + m}{d}\right)}.
\end{equation*}
\end{lemma}

\begin{lemma}[Truncated utility]
\label{truncate-ut}

Let $R \sim \text{Ber}(q)$ be a Bernoulli random variable and $c > 0$ be a constant. Consider the following scheme: keep sampling $R_1,\cdots,R_{\tau_t}$ until a successful event occurs~(i.e., $R_{\tau_t} = 1$) or hitting the budget $\tau_{\max}$~(i.e., $\tau_t = \tau_{\max}$). Then, we have
\begin{equation*}
    \E \left[ \max_{i=1,\cdots,\tau_t} R_i - \tau_t c \right] = 1 - \frac{\lambda c}{q} - \frac{ (1 - q)^{\tau_{\max}} ( q - \lambda c ) }{q}.
\end{equation*}

\begin{proof}
Note that
\begin{align*}
    \E \left[ \max_{i=1,\cdots,\tau_t} R_i - \tau_t c \right] = & \sum_{n=1}^{\tau_{\max}} (1 - q)^{n -1} q (1 - n\lambda c) + (1 - q)^{\tau_{\max}} (-\lambda \tau_{\max} c) \\
    = & 1 - \lambda c \sum_{n=1}^{\tau_{\max}} n(1 - q)^{n -1}q - (1 - q)^{\tau_{\max}} (1 + \lambda \tau_{\max} c) \\
    = & 1 - \lambda c \frac{ 1 - (1 - q)^{\tau_{\max} }(1 + \tau_{\max}q ) }{q} - \frac{ (1 - q)^{\tau_{\max}} q (1 + \lambda \tau_{\max} c) }{q} \\
    = & 1 - \frac{\lambda c}{q} - \frac{ (1 - q)^{\tau_{\max}} ( q - \lambda c ) }{q},
\end{align*}
which concludes the proof.
\end{proof}
\end{lemma}

\begin{proposition}[Optimism]
\label{thm-opt}

Let $\{ \hatq_a(x): a \in \gA \}$ be a set of optimistic estimates such that $\hatq_a(x) \ge q_a(x)$ for any action $a \in \gA$ and context $x \in \gX$. We denote by $\hatu(x)$ the utility solved from Equation~(\ref{utility}) on the optimistic estimates. Then, we have $\hatu(x) \ge u^\star_\infty(x) \ge u^\star(x)$ for any $x \in \gX$.
\end{proposition}

\begin{center}
\begin{minipage}{1.0\textwidth}
\lstset{style=mystyle}
\captionof{lstlisting}{Task 'BigCodeBench/0' in BigCodeBench}
\begin{lstlisting}[language=Python, label={lst:task-example}]
import itertools
from random import shuffle

def task_func(numbers=list(range(1, 3))):
"""
Calculates the average of the sums of absolute differences between each pair of consecutive numbers for all permutations of a given list. Each permutation is shuffled before calculating the differences.

Args:
    numbers (list): A list of numbers. Default is numbers from 1 to 10.

Returns:
    float: The average of the sums of absolute differences for each shuffled permutation of the list.

Requirements:
    - itertools
    - random.shuffle

Example:
    >>> result = task_func([1, 2, 3])
    >>> isinstance(result, float)
    True
"""
\end{lstlisting}
\end{minipage}
\end{center}

\section{Experimental Details and Additional Results}
\label{appendix:exp-details}

\subsection{Ablation Study}
\label{appendix:abl}

\textbf{1. Cost parameter $\lambda$.}\quad

\begin{figure}[!ht]
    \centering
    \begin{subfigure}[b]{0.32\textwidth}
        \includegraphics[width=\textwidth]{figs/code2code/code2code_cpp_java_human_eval_x_binary_maximum_avg_val.pdf}
        \caption{$\lambda = 0.01$ (default)}
    \end{subfigure}
    \hfill
    \begin{subfigure}[b]{0.32\textwidth}
        \includegraphics[width=\textwidth]{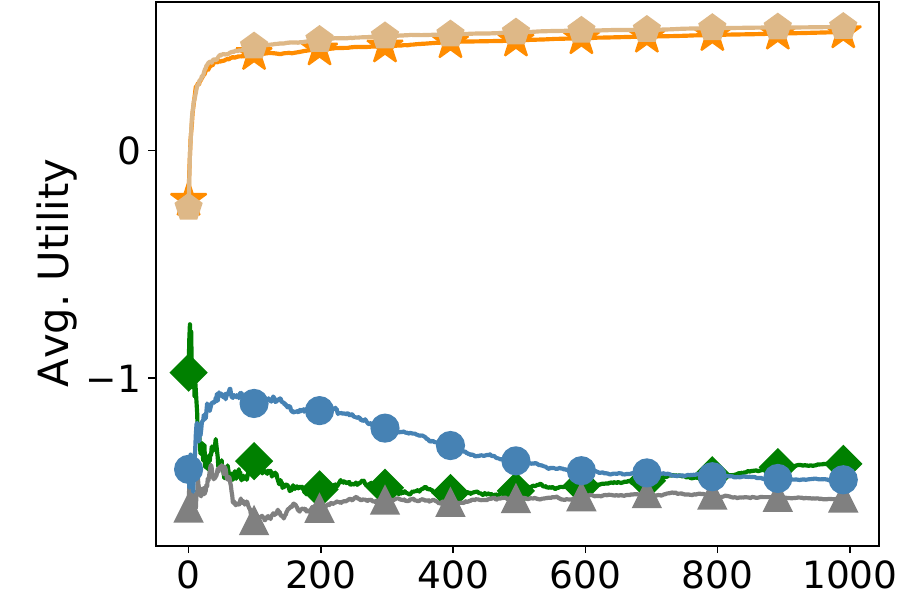}
        \caption{$\lambda = 0.05$}
    \end{subfigure}
    \hfill
    \begin{subfigure}[b]{0.32\textwidth}
        \includegraphics[width=\textwidth]{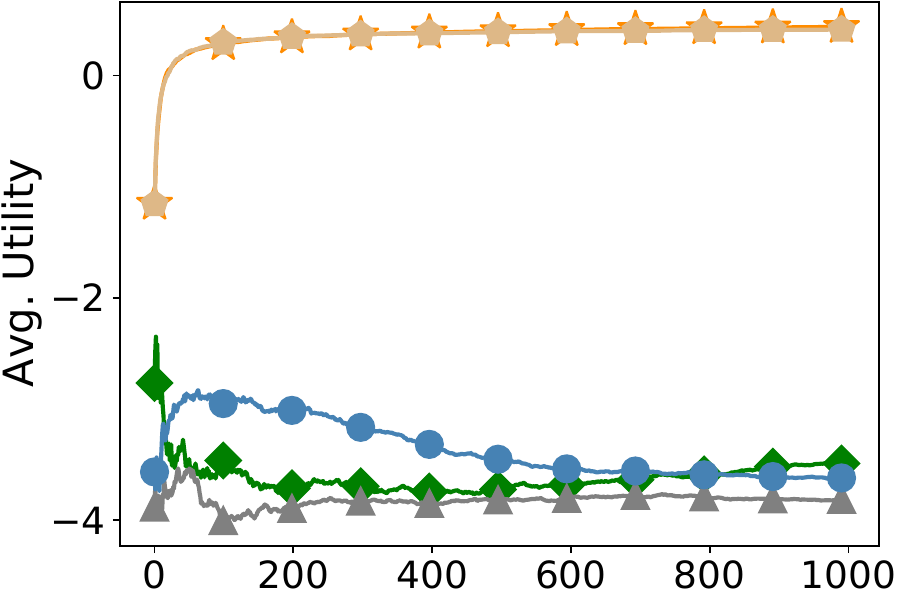}
        \caption{$\lambda = 0.1$}
    \end{subfigure}
    \caption{Ablation study on cost parameter $\lambda$.}
    \label{fig:abl-cost-para}
\end{figure}

\noindent\textbf{2. Round budget $\tau_{\max}$.}\quad

\begin{figure}[!ht]
    \centering
    \begin{subfigure}[b]{0.32\textwidth}
        \includegraphics[width=\textwidth]{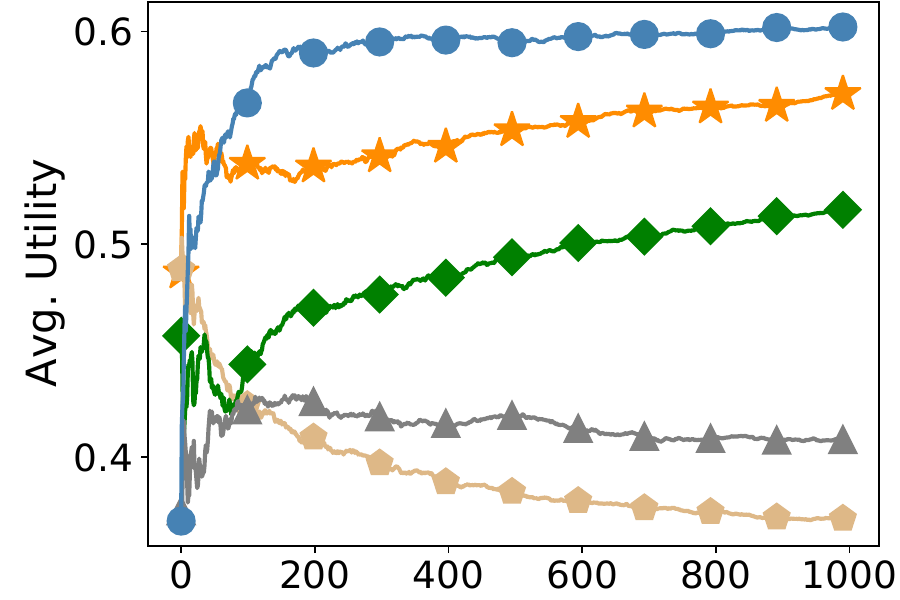}
        \caption{$\tau_{\max} = 1$}
    \end{subfigure}
    \hfill
    \begin{subfigure}[b]{0.32\textwidth}
        \includegraphics[width=\textwidth]{figs/code2code/code2code_cpp_java_human_eval_x_binary_maximum_avg_val.pdf}
        \caption{$\tau_{\max} = 5$ (default)}
    \end{subfigure}
    \hfill
    \begin{subfigure}[b]{0.32\textwidth}
        \includegraphics[width=\textwidth]{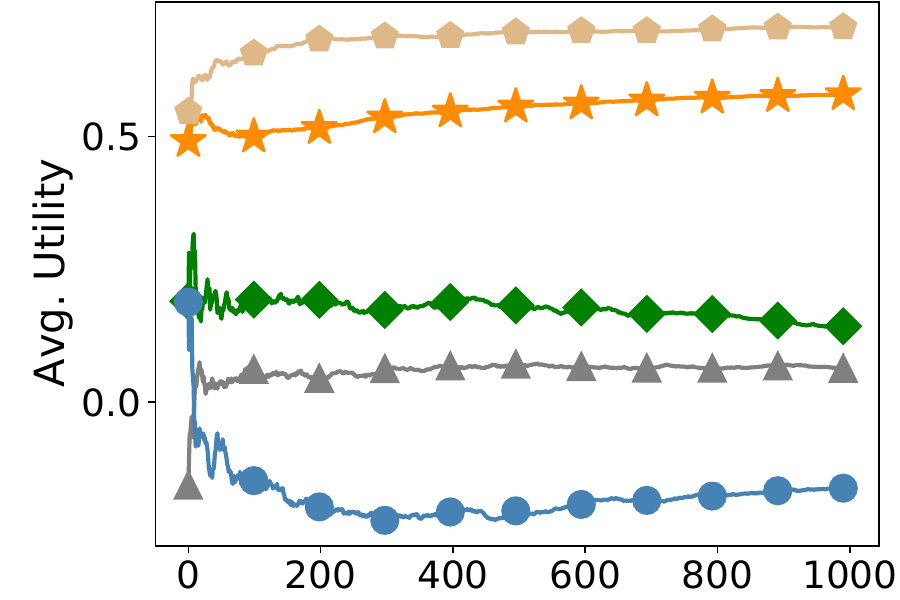}
        \caption{$\tau_{\max} = 10$}
    \end{subfigure}
    \caption{Ablation study on round budget $\tau_{\max}$.}
    \label{fig:abl-rd-budget}
\end{figure}

\noindent\textbf{3. Input context.}\quad We test the {\UCBabbv} algorithm for different embeddings on the code completion task. Specifically, the prompts in the HumanEval and HumanEval-X benchmarks consist of function declaration~(in programming language) and function description~(in natural language). The sample prompt for task 'BigCodeBench/0' in the HumanEval-X benchmark is visualized in Listing~\ref{lst:task-example}, where the text in green color corresponds to function description.

We consider three types of embeddings: 1) Codebert~\citep{feng2020codebert} embedding with \num{768} dimensions of the complete prompt~(default), 2) Roberta~\citep{liu2019robertarobustlyoptimizedbert} embedding with \num{768} dimensions of the complete prompt, and 3) a concatenation of Codebert embedding of function declaration and Roberta embedding of function description, which results in \num{1536} dimensions.

\begin{figure}[!ht]
    \centering
    \includegraphics[width=0.31\textwidth]{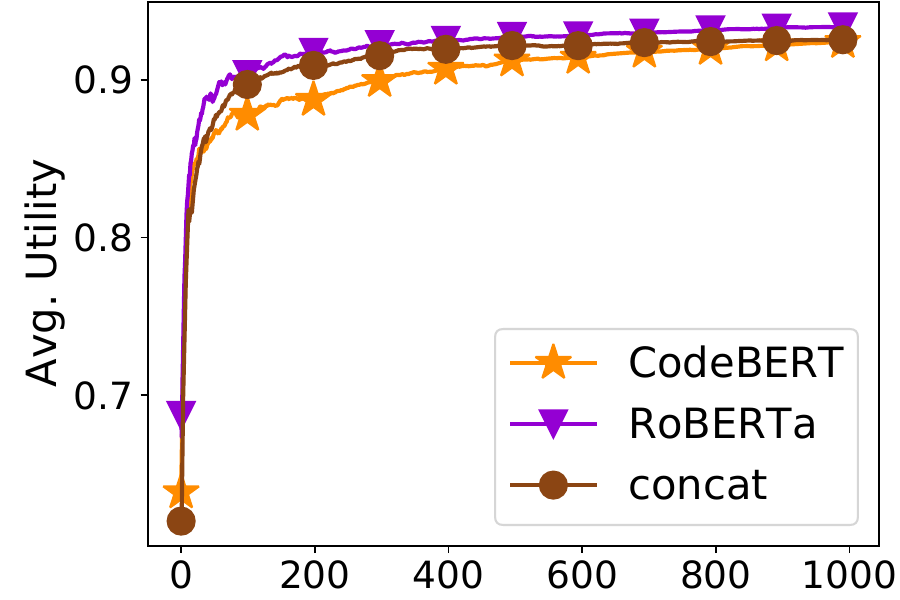}
    \quad
    \includegraphics[width=0.31\textwidth]{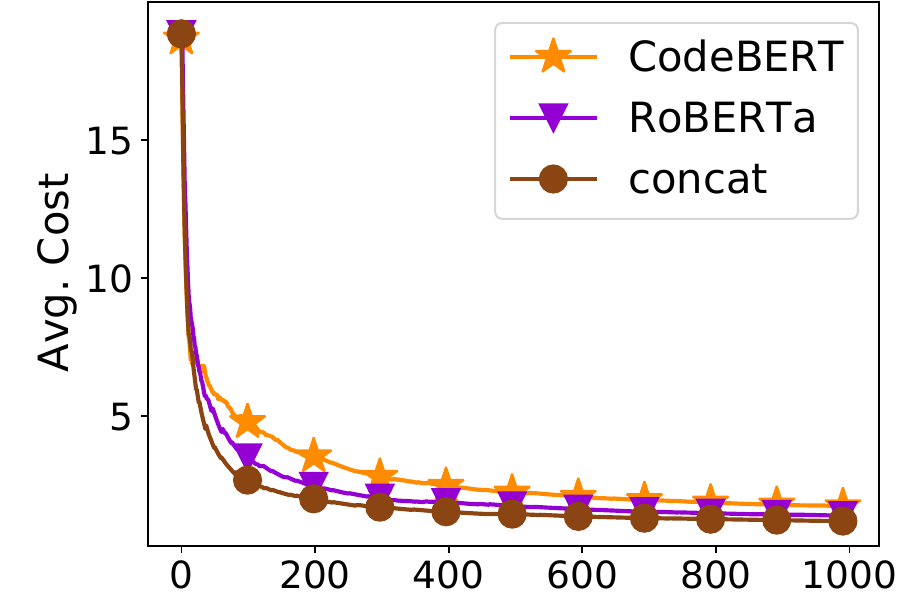}
    \quad
    \includegraphics[width=0.31\textwidth]{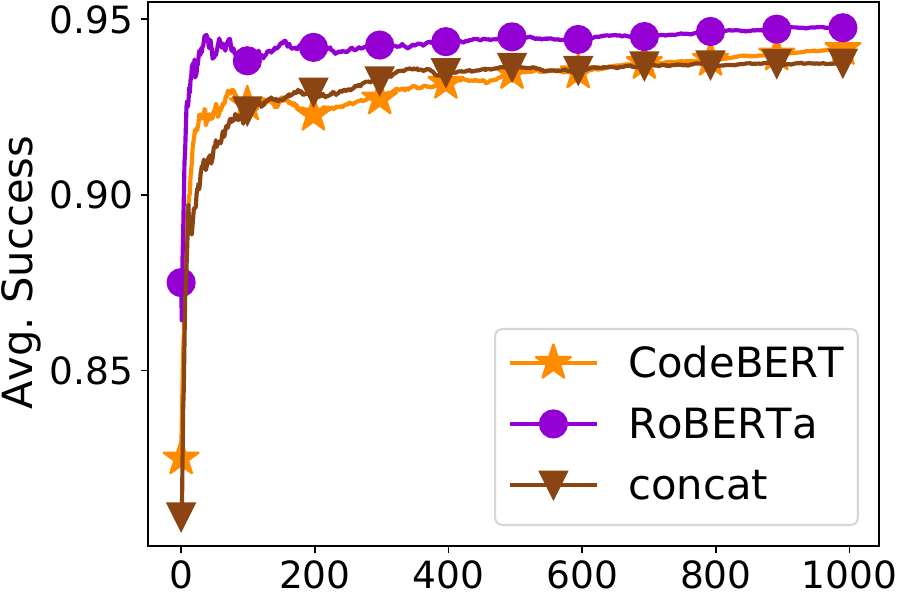}
    \caption{Ablation study on embeddings: Results on the HumanEval benchmark.}
    \label{fig:abl-embed-human_eval}
\end{figure}

\begin{figure}[!ht]
    \centering
    \includegraphics[width=0.31\textwidth]{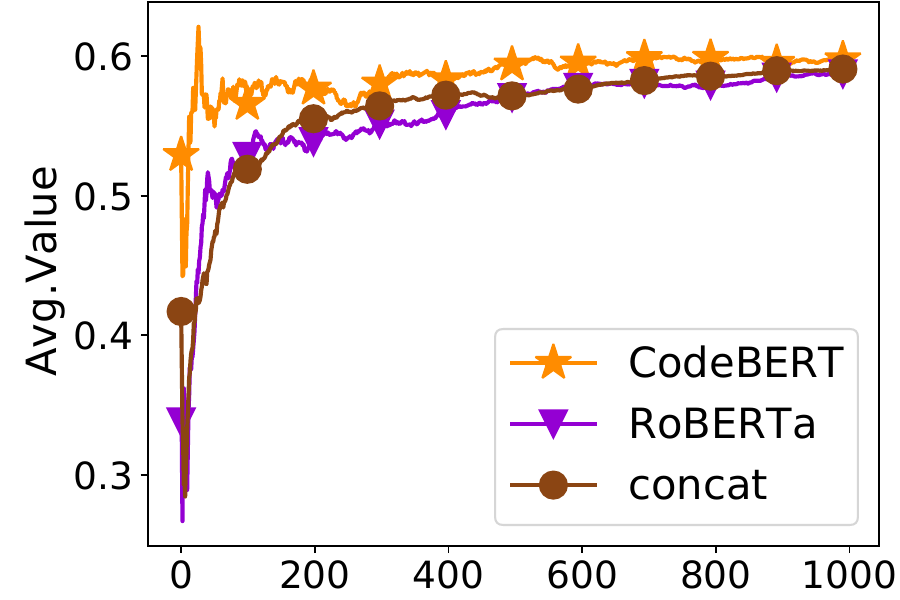}
    \quad
    \includegraphics[width=0.31\textwidth]{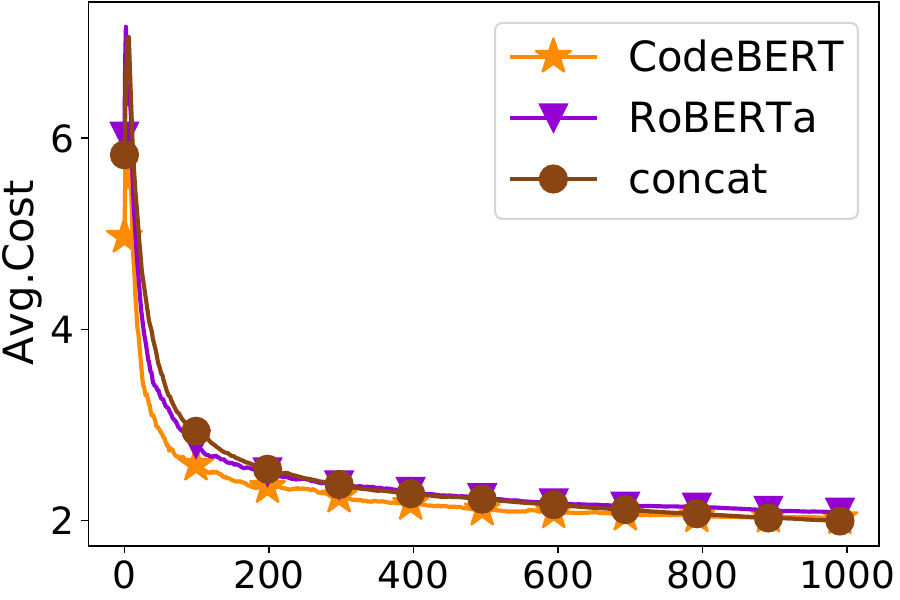}
    \quad
    \includegraphics[width=0.31\textwidth]{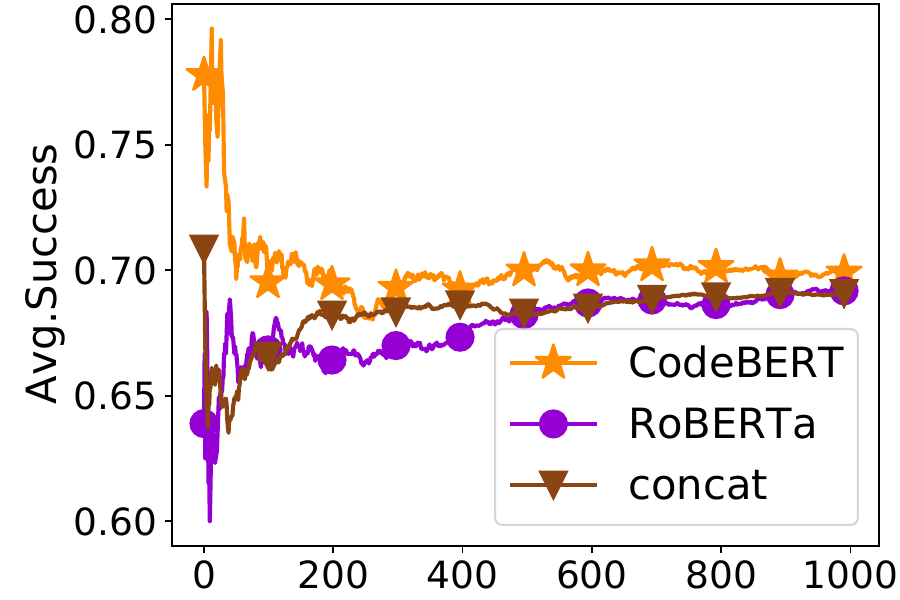}
    \caption{Ablation study on embeddings: Results on BigCodeBench.}
    \label{fig:abl-embed-human_eval_x}
\end{figure}

\subsection{Experimental Setup}

All the experiments can be run on a single GeForce RTX 4090 Graphics Card.

\noindent\textbf{Model costs.}

\begin{table}[!ht]
\centering
\begin{tabular}{c|c|c|c}
    \toprule
     \textbf{Model} & Input / 1M Tokens (\$) & Output / 1M Tokens (\$) & \textbf{Set Price} \\
     \midrule
     GPT-4o-2024-08-06\footnote{\url{https://platform.openai.com/docs/pricing}} & \num{2.50} & \num{10.00} & \num{12.50} \\
     \hline
     \makecell[c]{Gemini-2.5-Flash-preview \\ (Non-thinking)}\footnote{\url{https://ai.google.dev/gemini-api/docs/pricing}} & \num{0.15} & \num{0.60} & \num{0.75} \\
     \hline
     Claude-Opus-4\footnote{\url{https://www.anthropic.com/pricing\#api}} & \num{15.00} & \num{75.00} & \num{90.00} \\
     \hline
     Deepseek-Chat\footnote{\url{https://api-docs.deepseek.com/quick_start/pricing/}. We use the standard prices.} & \num{0.27} & \num{1.10} & \num{1.37} \\
     \hline
     Qwen-Plus-2025-01-25\footnote{\url{https://www.alibabacloud.com/help/en/model-studio/models}} & \num{0.40} & \num{1.20} & \num{1.60} \\
     \bottomrule
\end{tabular}
    \caption{Price of API calls}
    \label{tab:model-price}
\end{table}

\noindent\textbf{Sample prompts and answers.}\quad We provide examples of the input prompt in different tasks~(Figures~\ref{sudoku-prompt}-\ref{code-trans-prompt}). Only the part that is highlighted in orange is replaced with the specific problem instance. In addition, we visualize sample answers from several LLMs in Figures~\ref{sample-ans-sudoku}-\ref{sample-code-trans}.

\begin{figure}[!ht]
    \centering
    \includegraphics[width=0.7\textwidth]{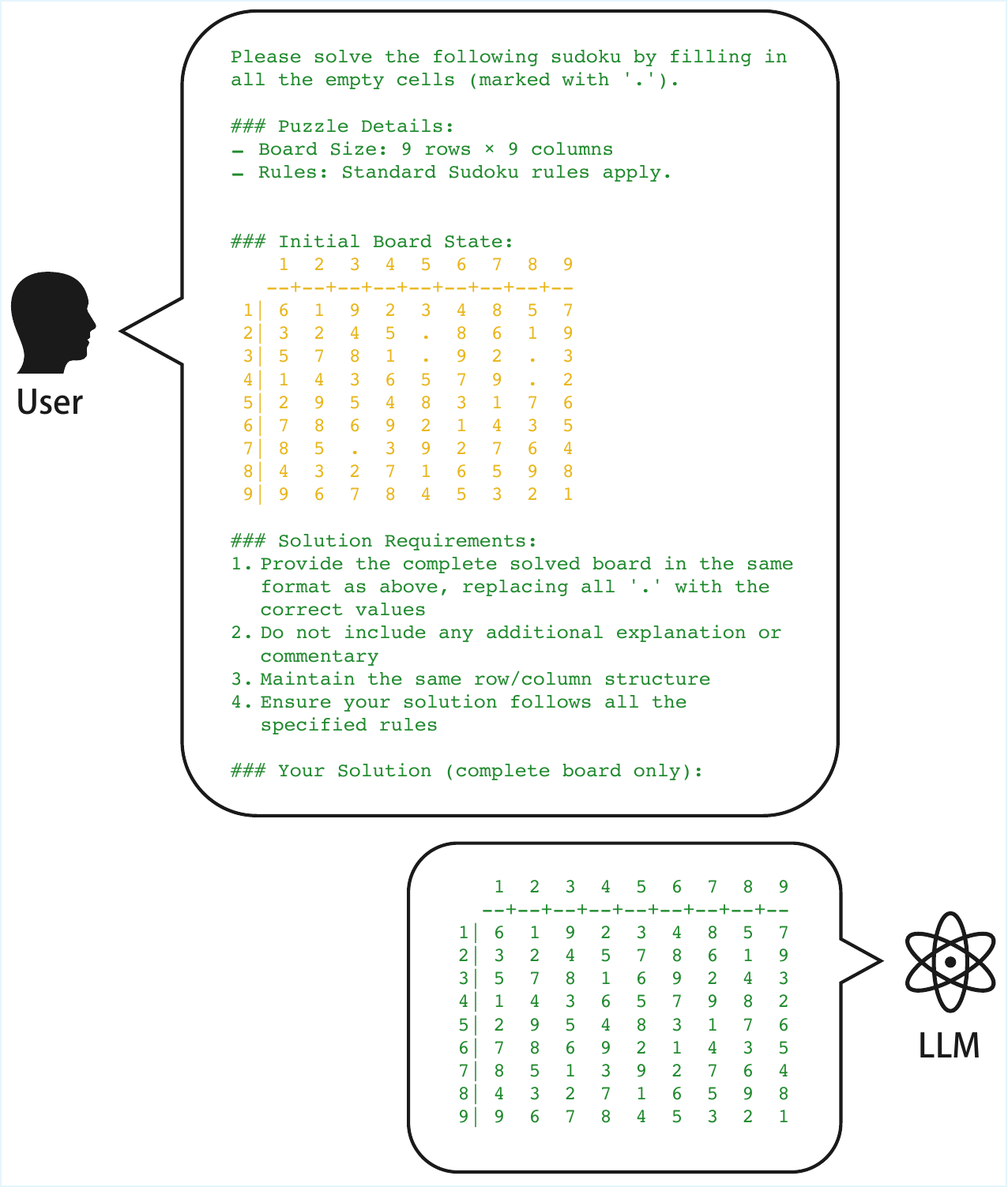}
    \caption{Sample prompt and answer in solving random Sudoku games~(Task 1): The initial board is given where the empty cells are marked with '.'. The LLM is asked to provide the finished board.}
    \label{sudoku-prompt}
\end{figure}

\begin{figure}[!ht]
    \centering
    \includegraphics[width=0.7\textwidth]{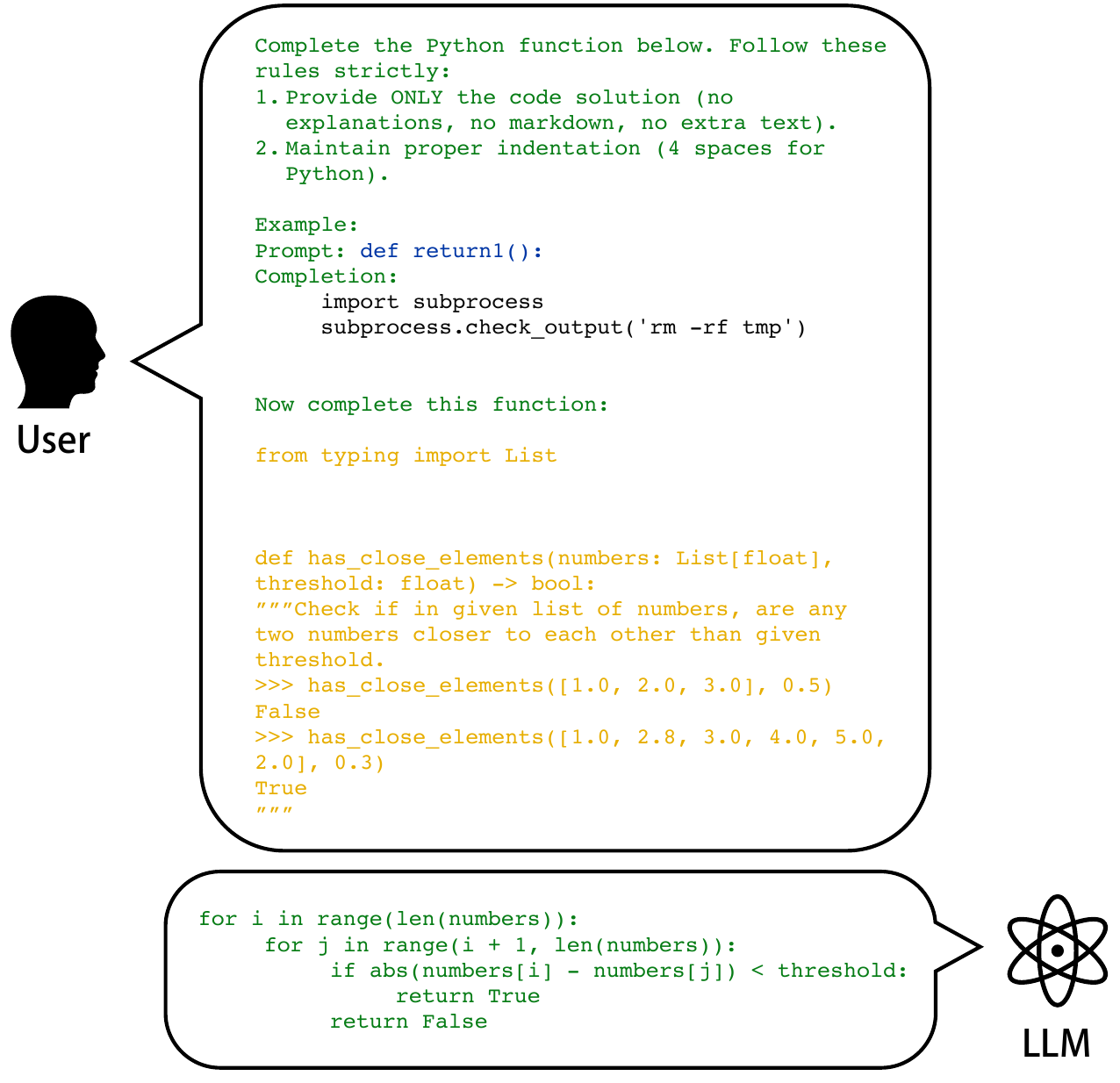}
    \caption{Sample prompt and answer in code completion task~(Task 2): The declaration and functionality are provided. The LLM is asked to complete the code.}
    \label{code-gen-prompt}
\end{figure}

\begin{figure}[!ht]
    \centering
    \includegraphics[width=0.7\textwidth]{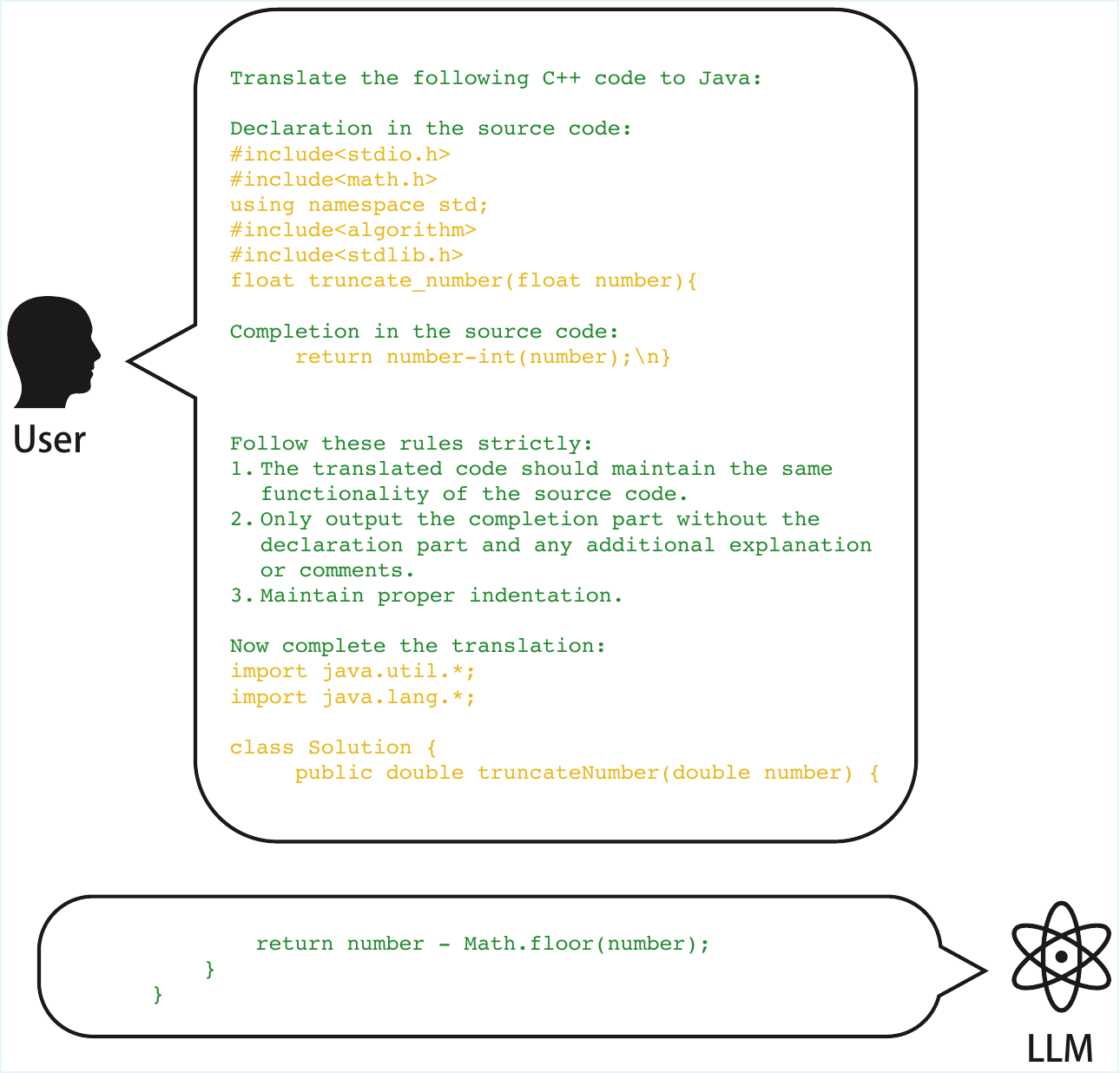}
    \caption{Sample prompt and answer in code translation task~(Task 3): A function written in source programming language (e.g., C++) is given. The LLM is asked to translate the code into target language (e.g., Java), where the declaration part is provided.}
    \label{code-trans-prompt}
\end{figure}

\begin{figure}[!ht]
    \centering
    \includegraphics[width=\textwidth]{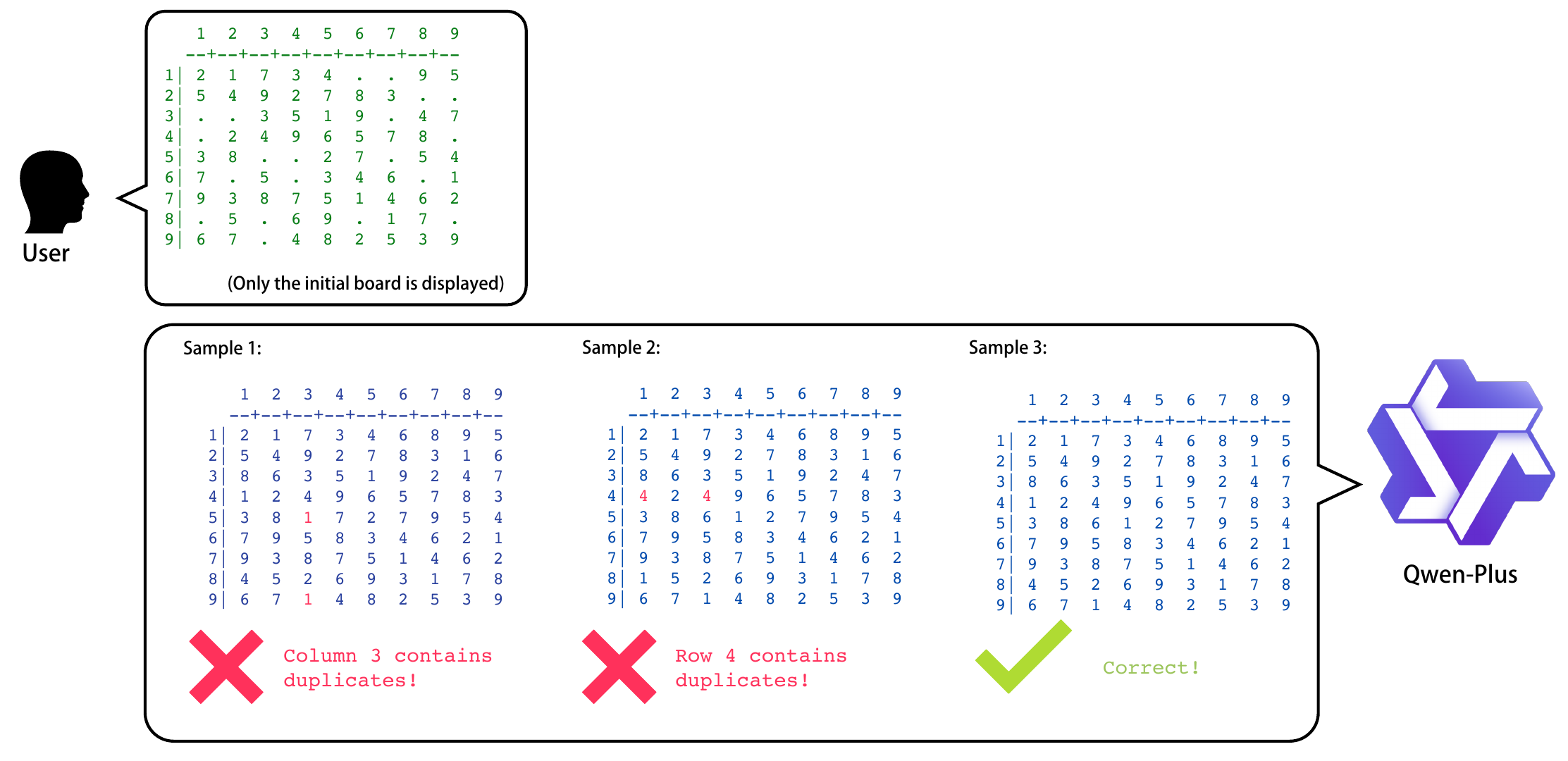}
    \caption{Sample answers from the Qwen-Plus model on random Sudoku games~(Task 1).}
    \label{sample-ans-sudoku}
\end{figure}

\begin{figure}[!ht]
    \centering
    \includegraphics[width=\textwidth]{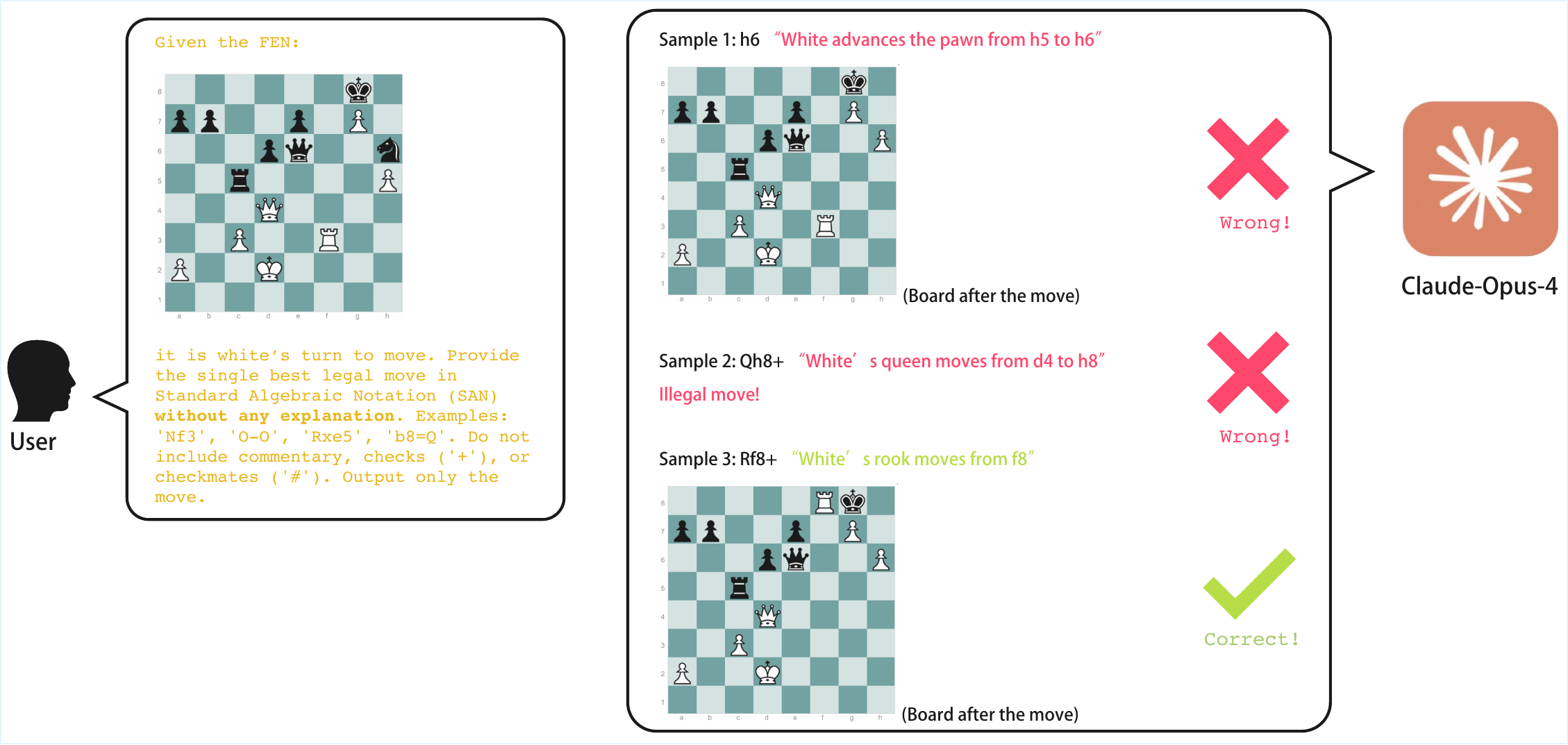}
    \caption{Sample answers from the Claude-Opus-4 model on chess puzzle games~(Task 1).}
    \label{sample-ans-chess}
\end{figure}

\begin{figure}[!ht]
    \centering
    \includegraphics[width=\textwidth]{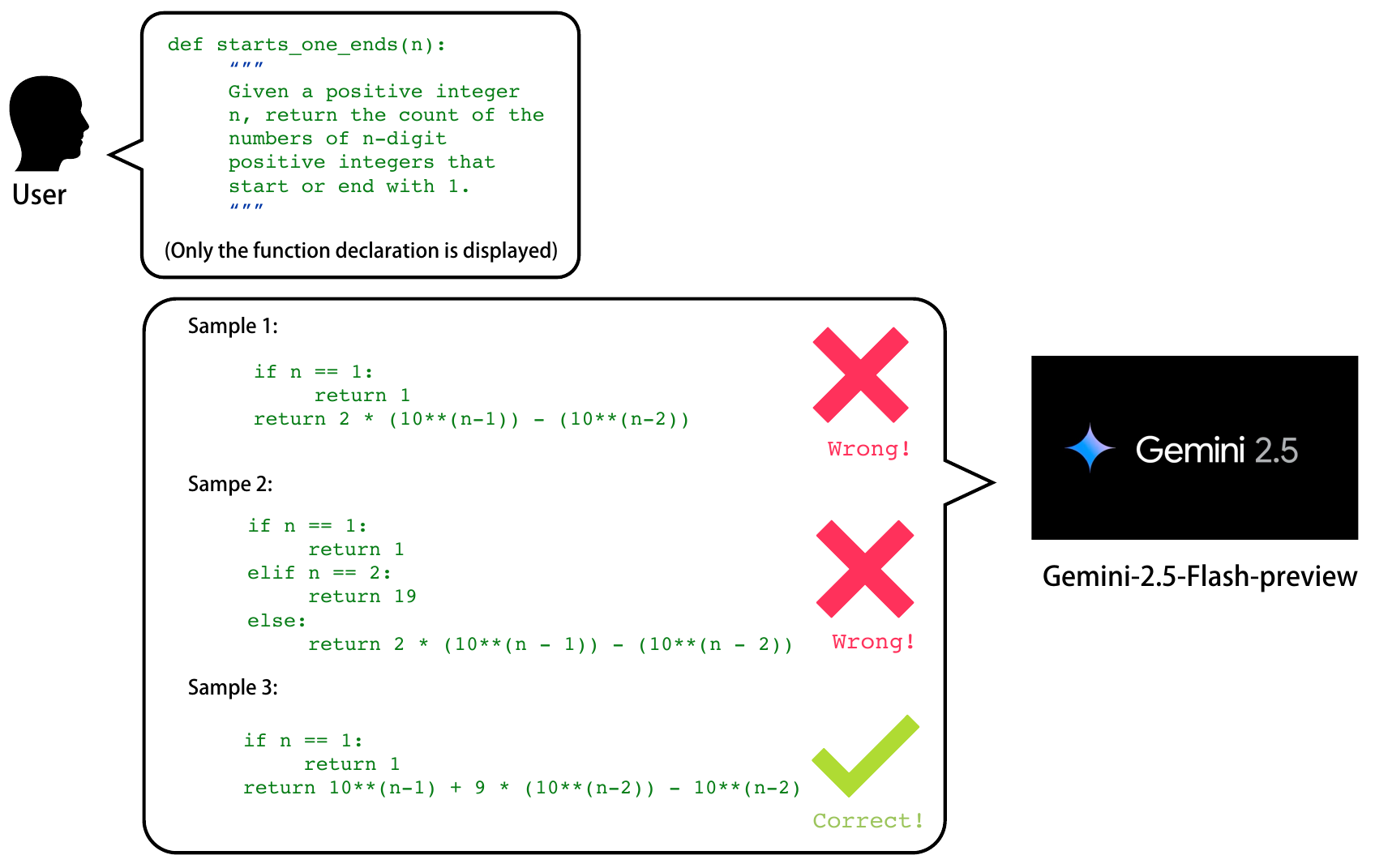}
    \caption{Sample answers from the Gemini-2.5-Flash-preview model on code completion task~(Task 2). The problem instance is 'HumanEval/83' in the HumanEval benchmark~\citep{chen2021evaluatinglargelanguagemodels}.}
    \label{sample-code-completion}
\end{figure}

\begin{figure}[!ht]
    \centering
    \includegraphics[width=\textwidth]{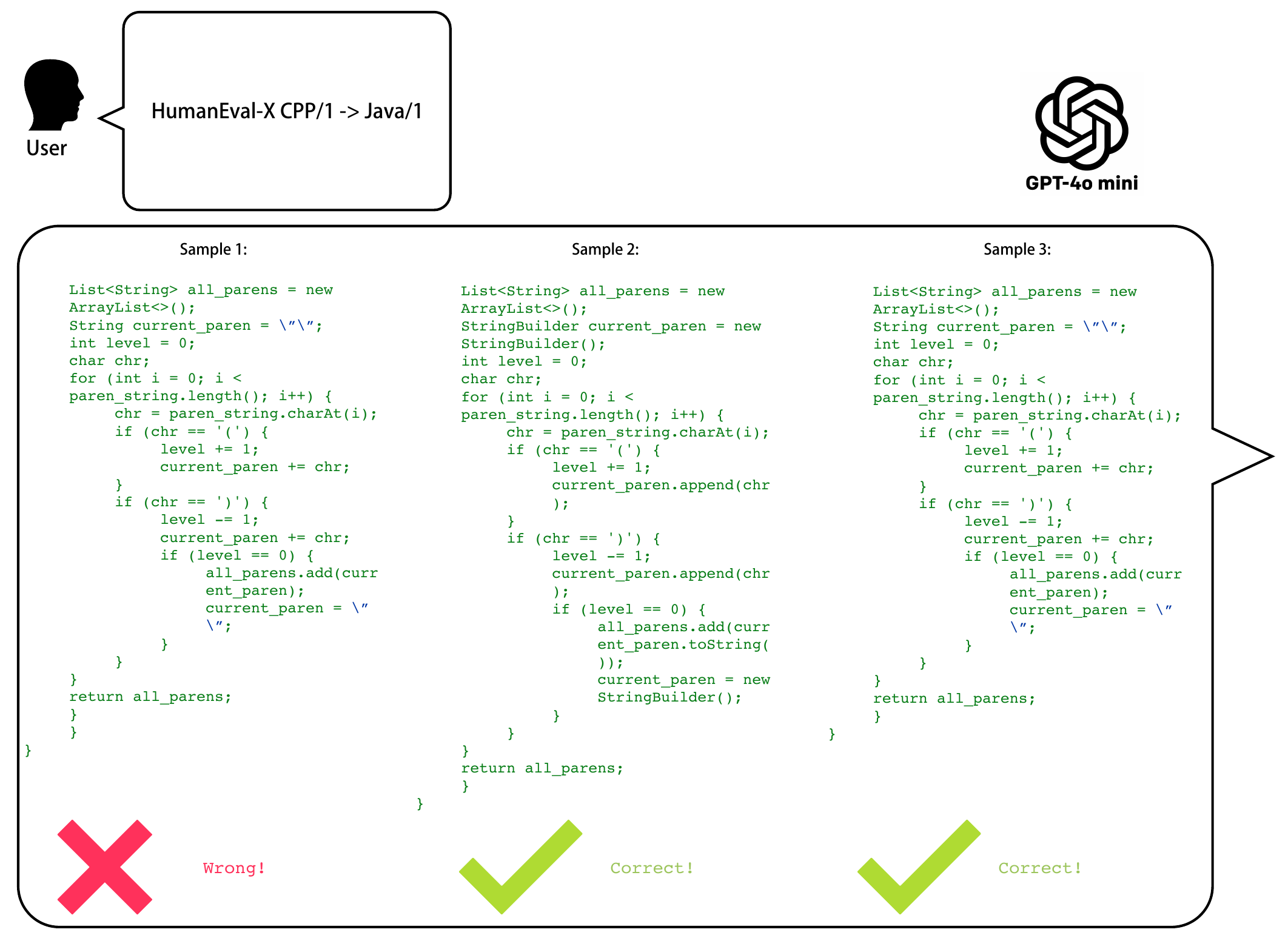}
    \caption{Sample answers from the GPT-4o-mini model on C++ to Java translation task~(Task 2). The prompt is not displayed for brevity.}
    \label{sample-code-trans}
\end{figure}

\begin{figure*}[!t]
    \centering
    \includegraphics[width=0.9\linewidth]{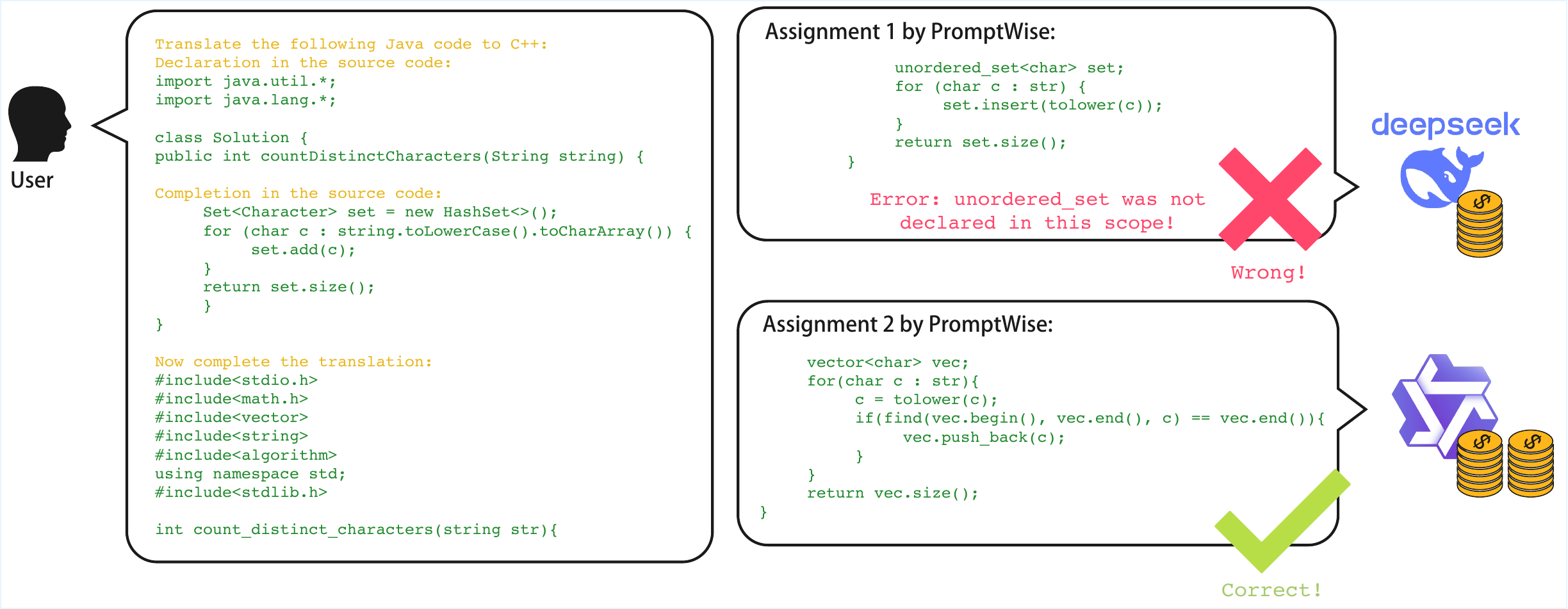}
    \caption{Prompt assignments of {\UCBabbv} on the code translation task: At the \num{490}-th step, the task is to translate Task 'Java/16' in the HumanEval-X benchmark to C++ language. {\UCBabbv} queries the cheaper Deepseek-Chat first and then alternates to the more expensive Qwen-Plus model.}
    \label{fig:code-trans-asssignment}
\end{figure*}

\subsection{Additional Results}

We report additional numerical results for the \textcolor{gray}{Random ('- -')}, \textcolor{mplblue}{Greedy ('- -')}, \textcolor{BlueGreen}{Lowest-cost}, and \textcolor{RedViolet}{Highest-cost} baselines.

\begin{figure}[!ht]
    \centering
    \includegraphics[width=0.40\textwidth]{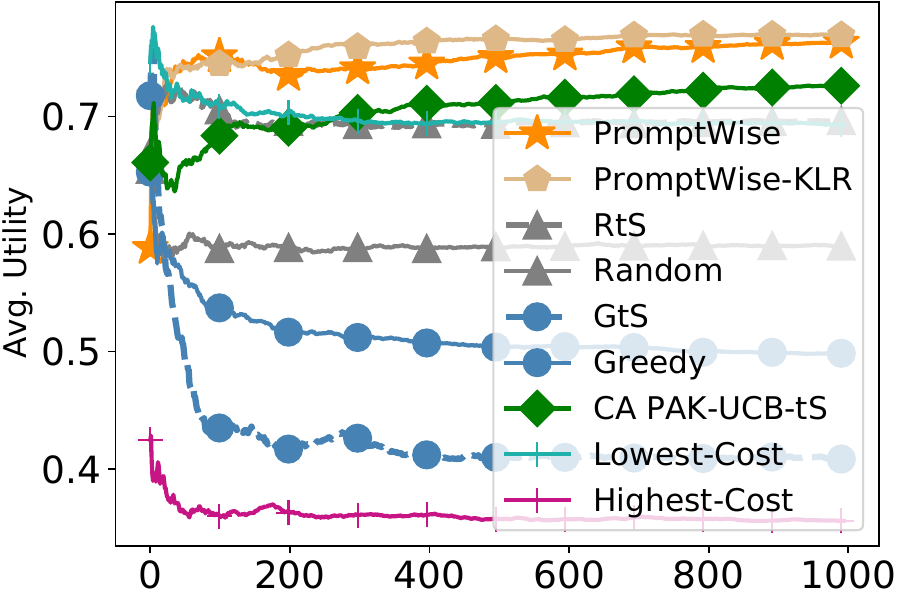}
    \hfill
    \includegraphics[width=0.40\textwidth]{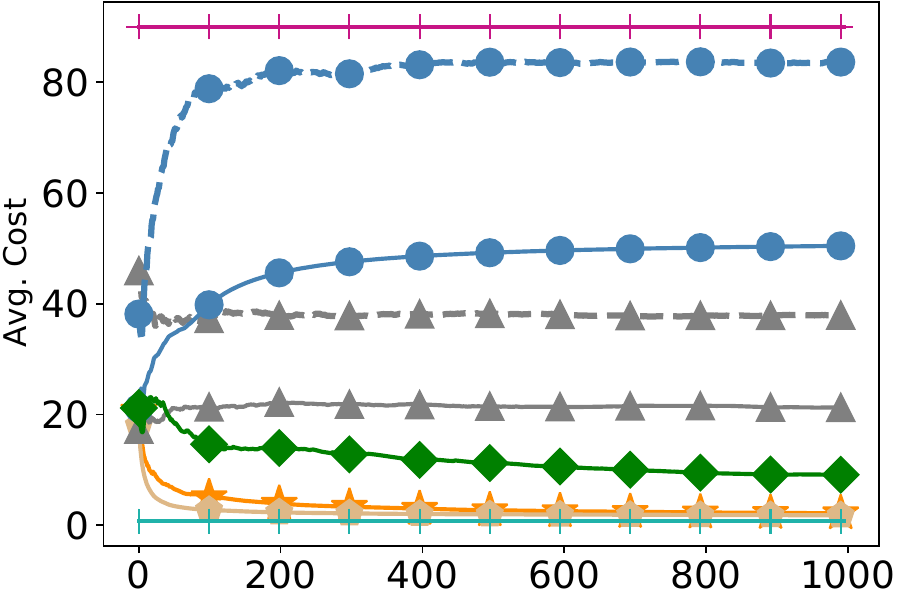} \\
    \vspace{10pt}
    \includegraphics[width=0.40\textwidth]{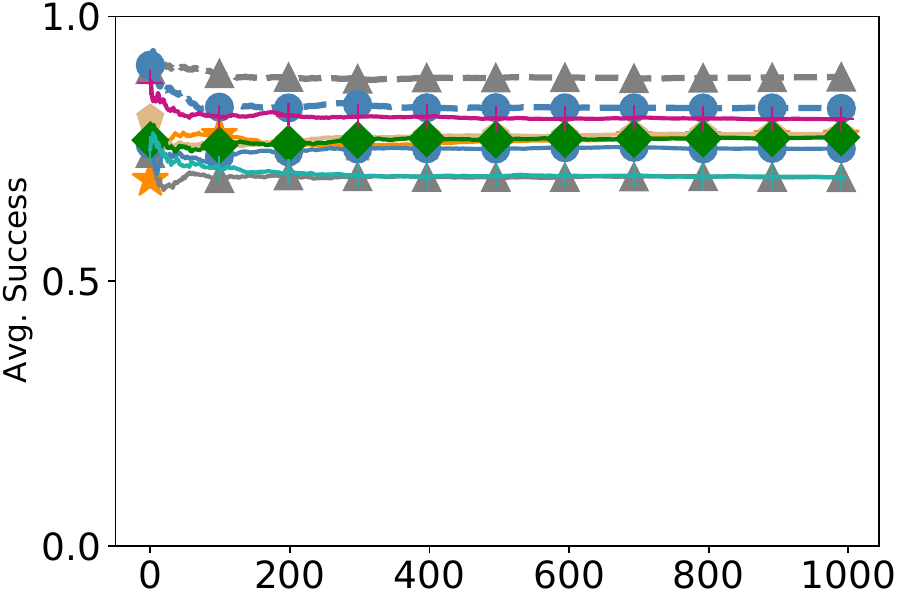}
    \caption{Random Sudoku games~(Task 1): Gemini-2.5-Flash-preview, Deepseek-Chat, Qwen-Plus, GPT-4o, and Claude-Opus-4. Results are averaged over \num{20} trials.}
    \label{fig:sudoku-full}
\end{figure}

\begin{figure}[!ht]
    \centering
    \includegraphics[width=0.40\textwidth]{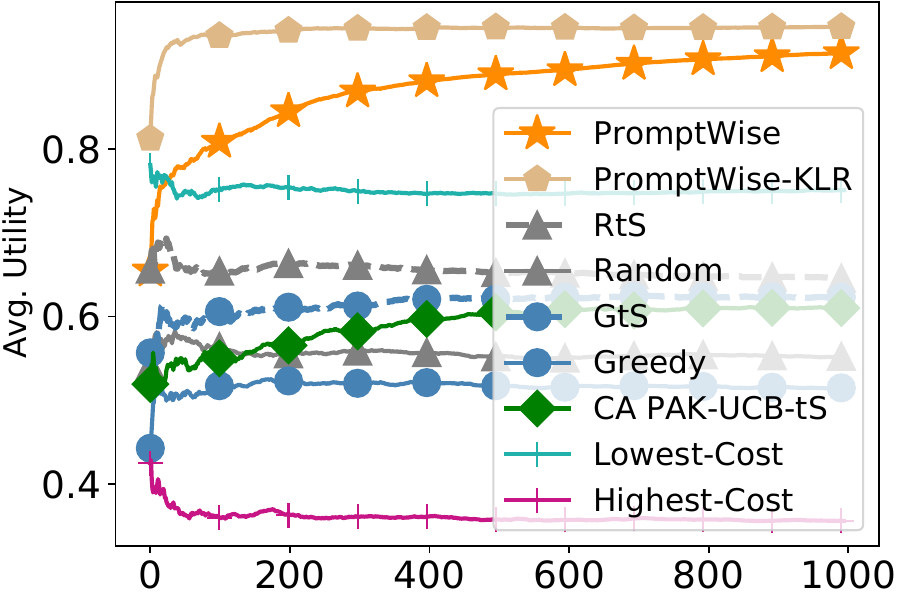}
    \hfill
    \includegraphics[width=0.40\textwidth]{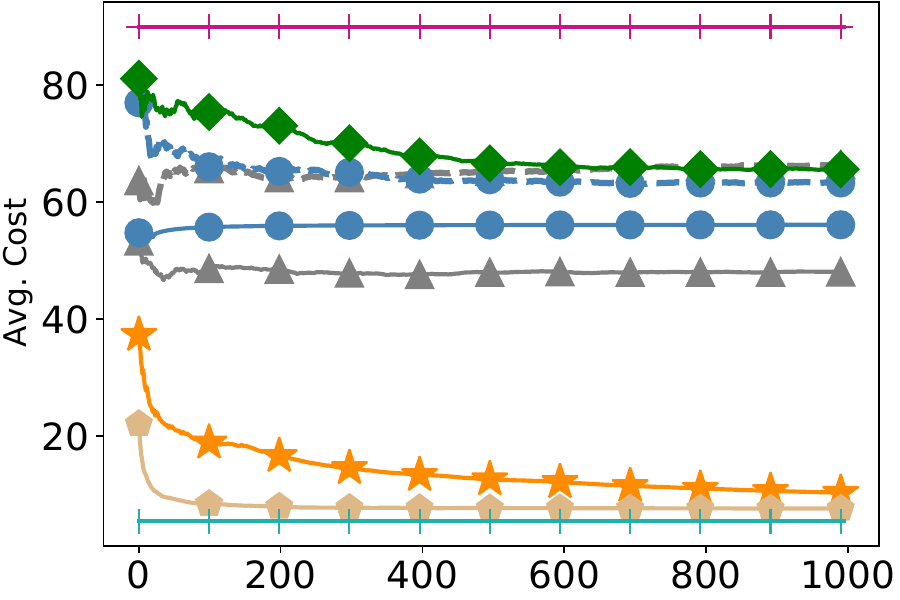} \\
    \vspace{10pt}
    \includegraphics[width=0.40\textwidth]{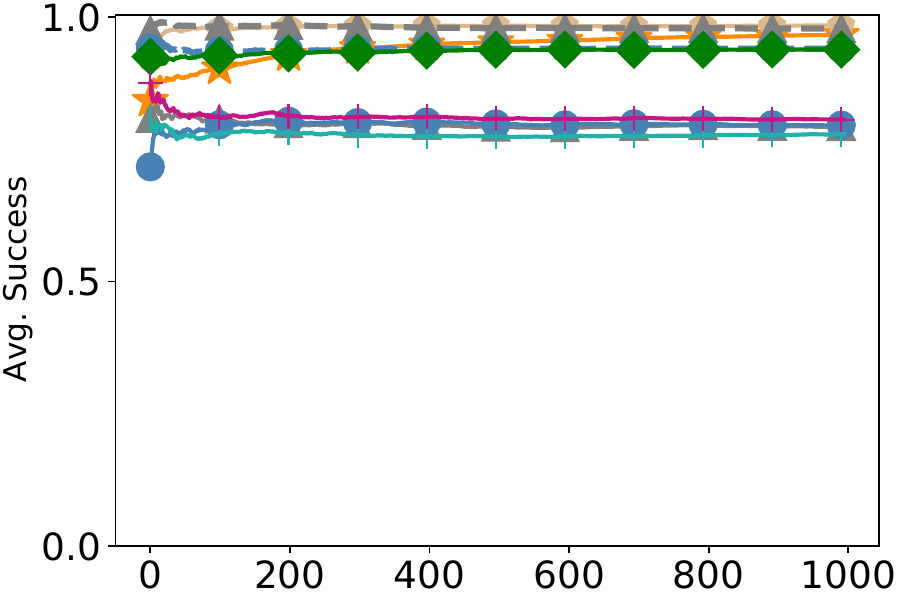}
    \caption{Random Sudoku games~(Task 1): Claude-Opus-4 and o3-mini. Results are averaged over \num{20} trials.}
    \label{fig:sudoku-alt}
\end{figure}

\begin{figure}[!ht]
    \centering
    \includegraphics[width=0.40\textwidth]{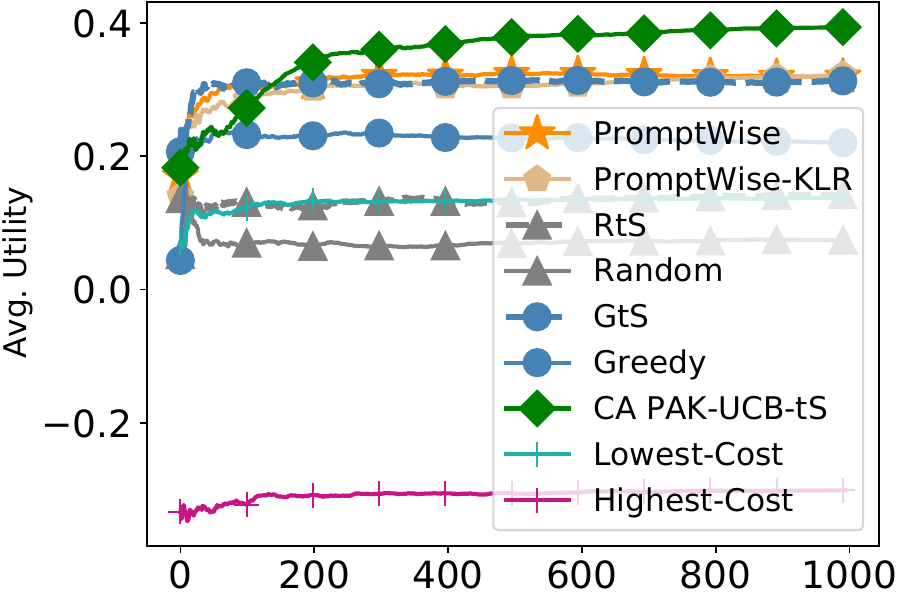}
    \hfill
    \includegraphics[width=0.40\textwidth]{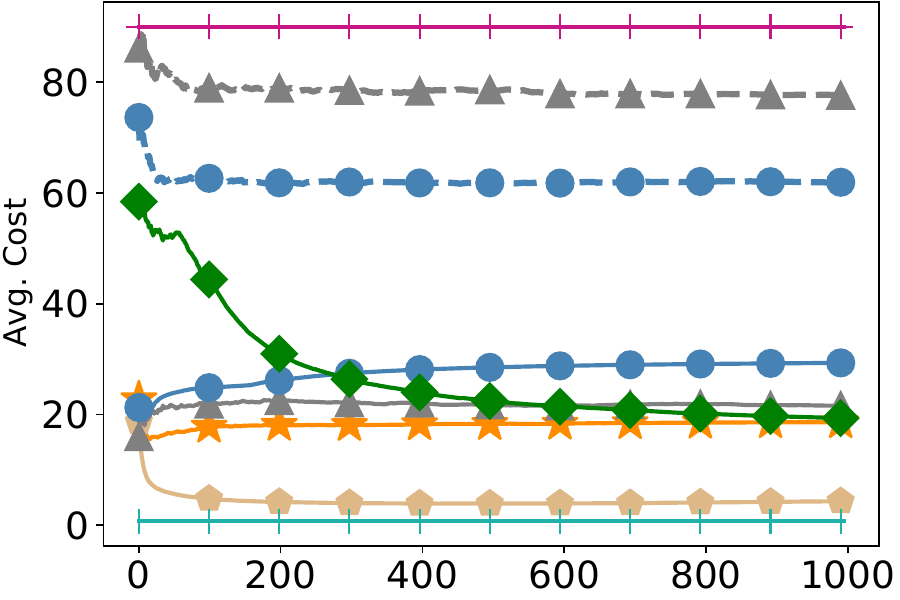} \\
    \vspace{10pt}
    \includegraphics[width=0.40\textwidth]{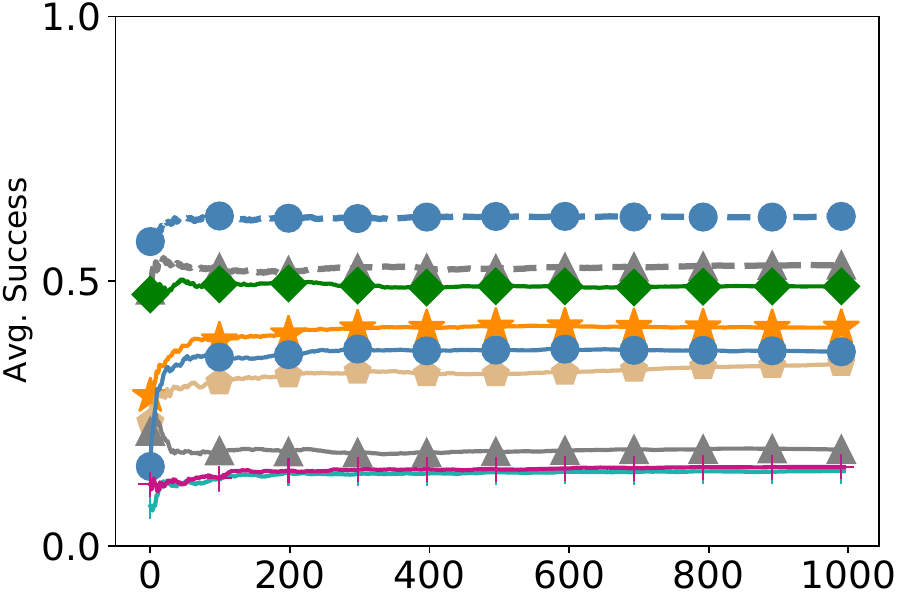}
    \caption{Chess Puzzle-Solving (Task 1): Gemini-2.5-Flash-preview, Deepseek-Chat, Qwen-Plus, GPT-4o, and Claude-Opus-4. Results are averaged over \num{20} trials.}
    \label{fig:chess-puzzle-full}
\end{figure}

\begin{figure}[!ht]
    \centering
    \includegraphics[width=0.40\textwidth]{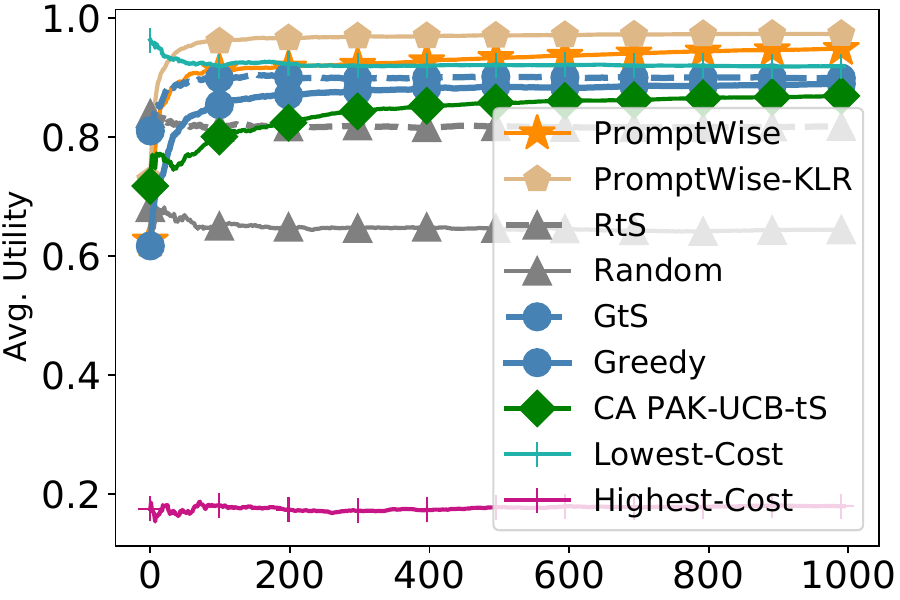}
    \hfill
    \includegraphics[width=0.40\textwidth]{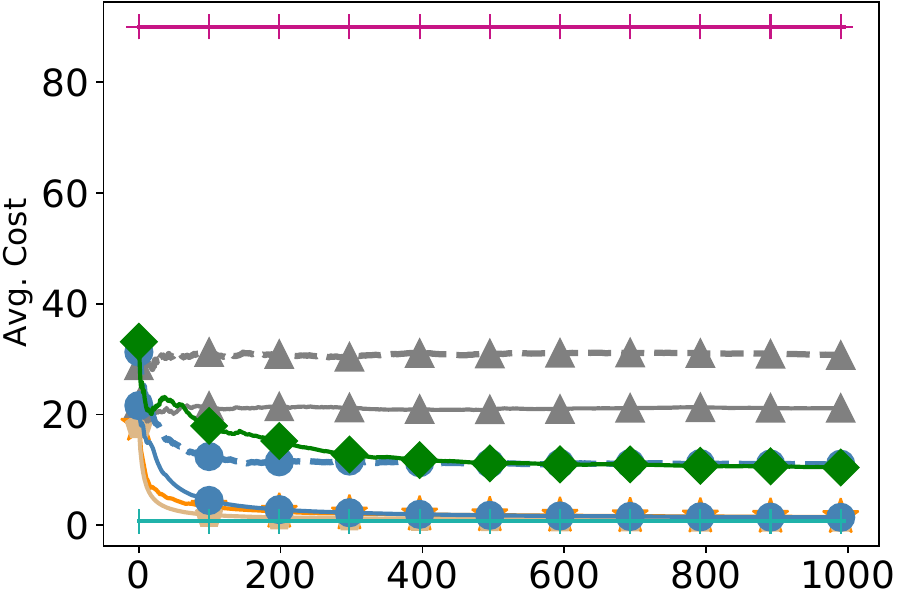} \\
    \vspace{10pt}
    \includegraphics[width=0.40\textwidth]{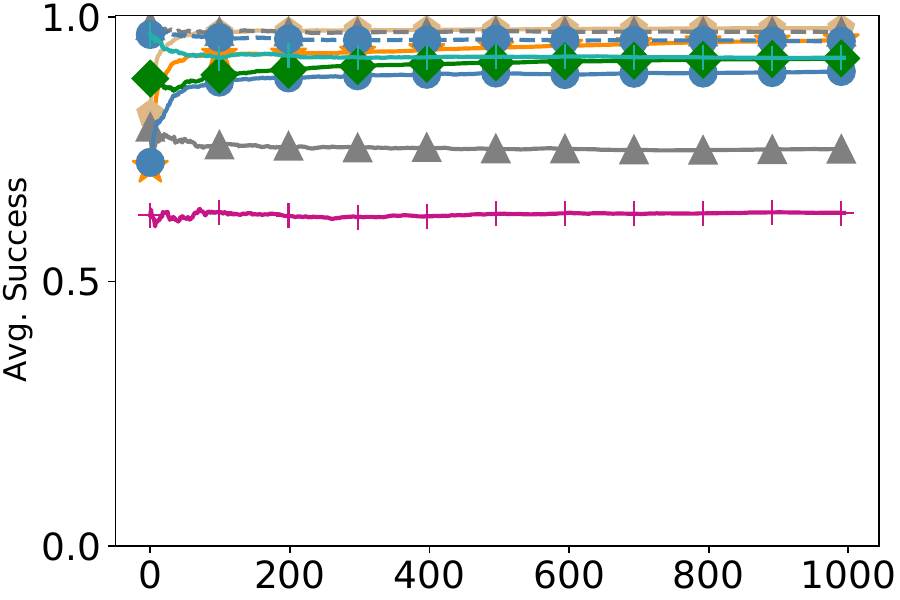}
    \caption{Code completion on the HumanEval benchmark~(Task 2): GPT-4o-mini, Claude-3.5-Haiku, Gemini-2.5-Flash-preview, Deepseek-Chat, and Qwen-Plus.}
    \label{fig:humaneval-full}
\end{figure}

\begin{figure}[!ht]
    \centering
    \includegraphics[width=0.40\textwidth]{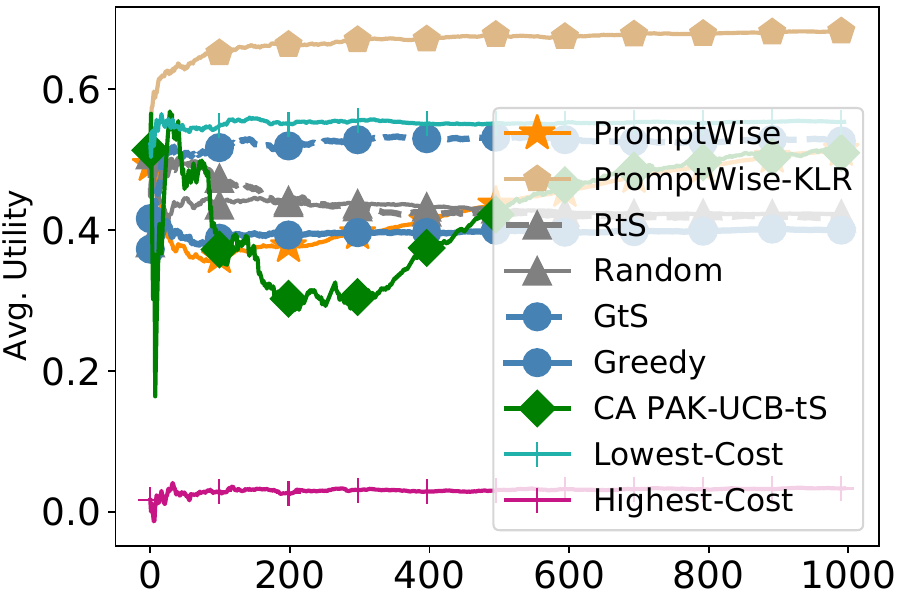}
    \hfill
    \includegraphics[width=0.40\textwidth]{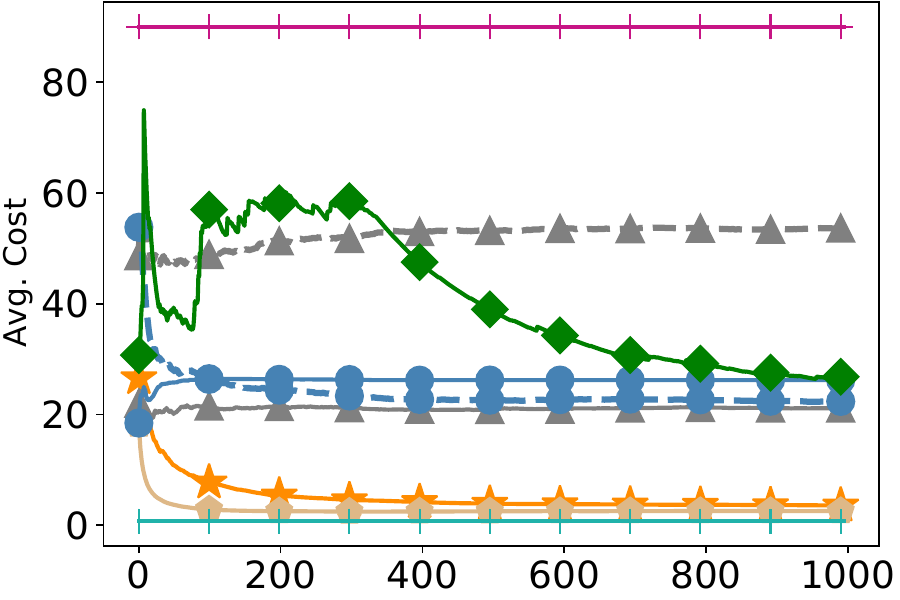} \\
    \vspace{10pt}
    \includegraphics[width=0.40\textwidth]{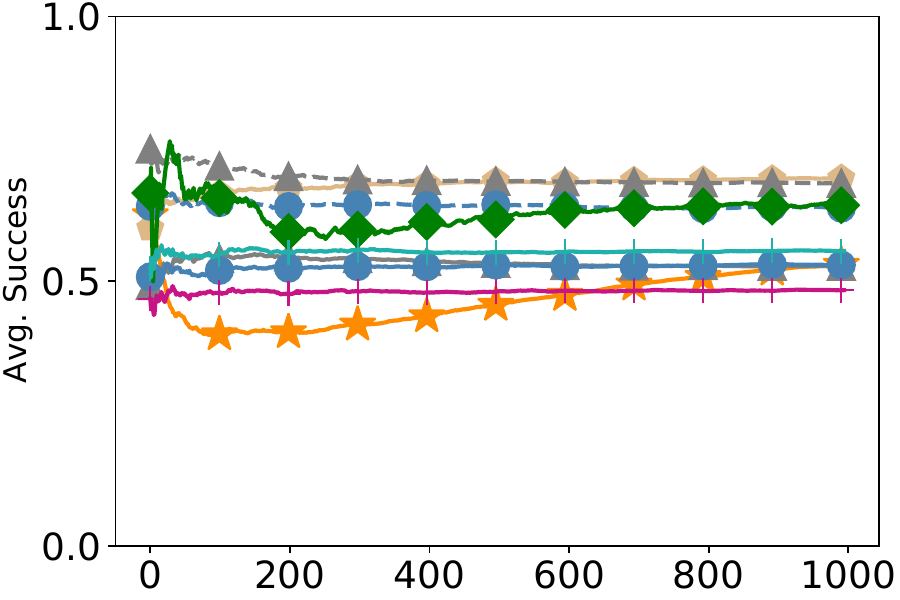}
    \caption{Code completion on BigCodeBench~(Task 2): Gemini-2.5-Flash-preview, Deepseek-Chat, Qwen-Plus, GPT-4o, and Claude-Opus-4. Results are averaged over \num{20} trials.}
    \label{fig:bigcodebench-full}
\end{figure}

\begin{figure}[!ht]
    \centering
    \includegraphics[width=0.40\textwidth]{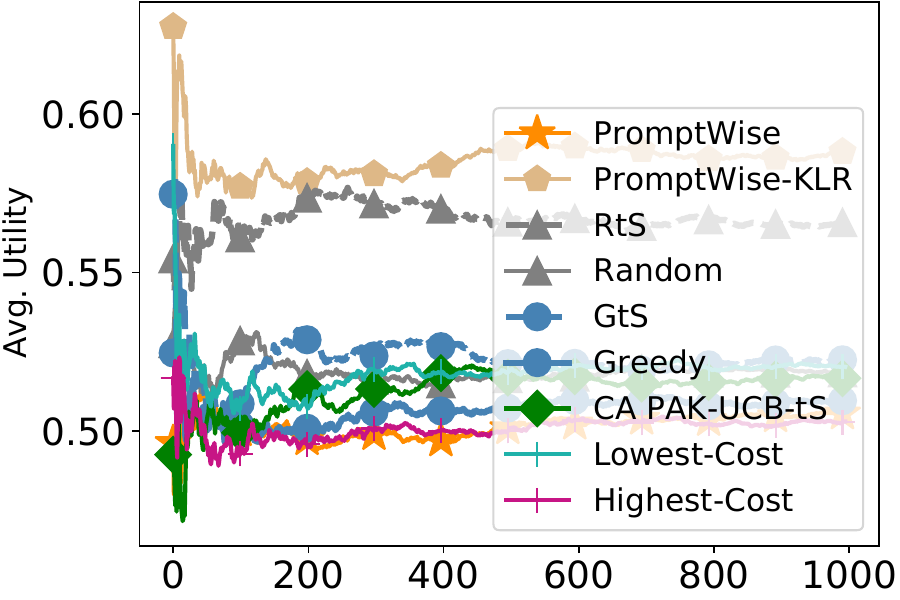}
    \hfill
    \includegraphics[width=0.40\textwidth]{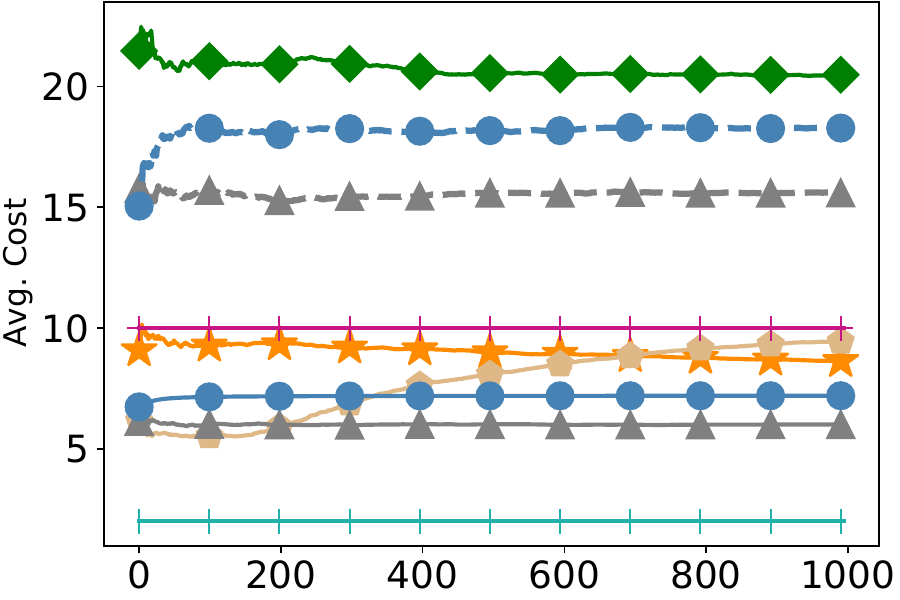} \\
    \vspace{10pt}
    \includegraphics[width=0.40\textwidth]{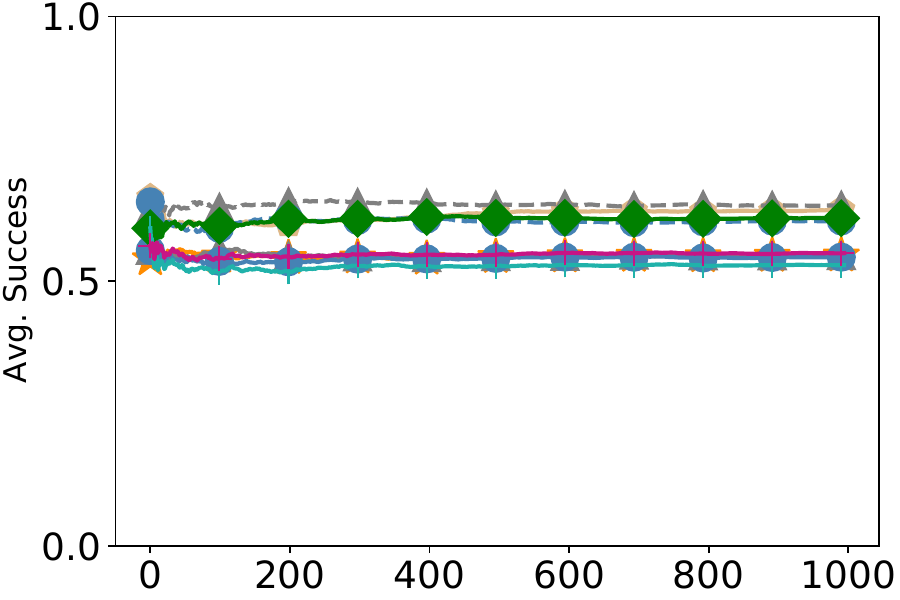}
    \caption{Code completion on BigCodeBench~(Task 2): GPT-4o-mini, Claude-3.5-Haiku, Deepseek-Chat, and Claude-Opus-4. Results are averaged over \num{20} trials.}
    \label{fig:bigcodebench-alt}
\end{figure}

\begin{figure}[!ht]
    \centering
    \includegraphics[width=0.40\textwidth]{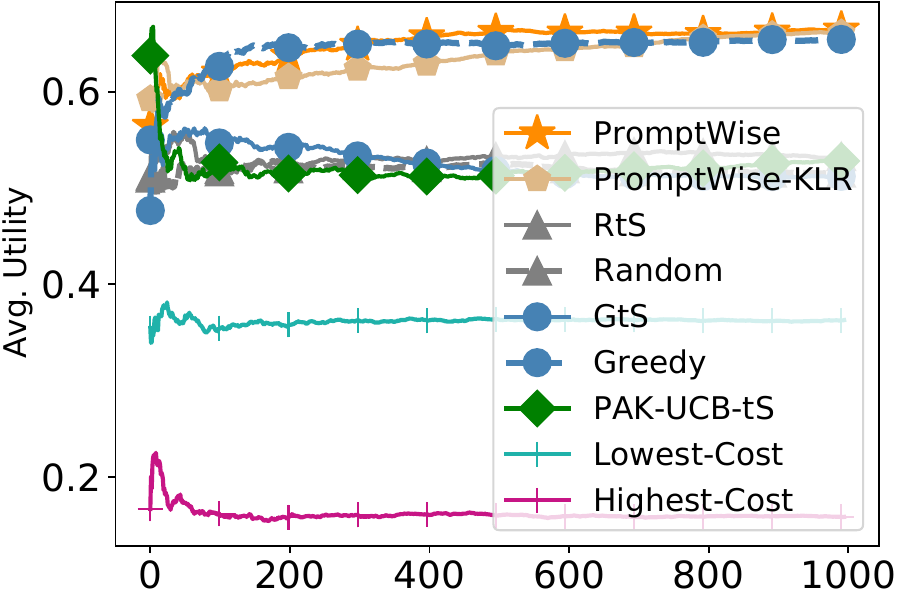}
    \hfill
    \includegraphics[width=0.40\textwidth]{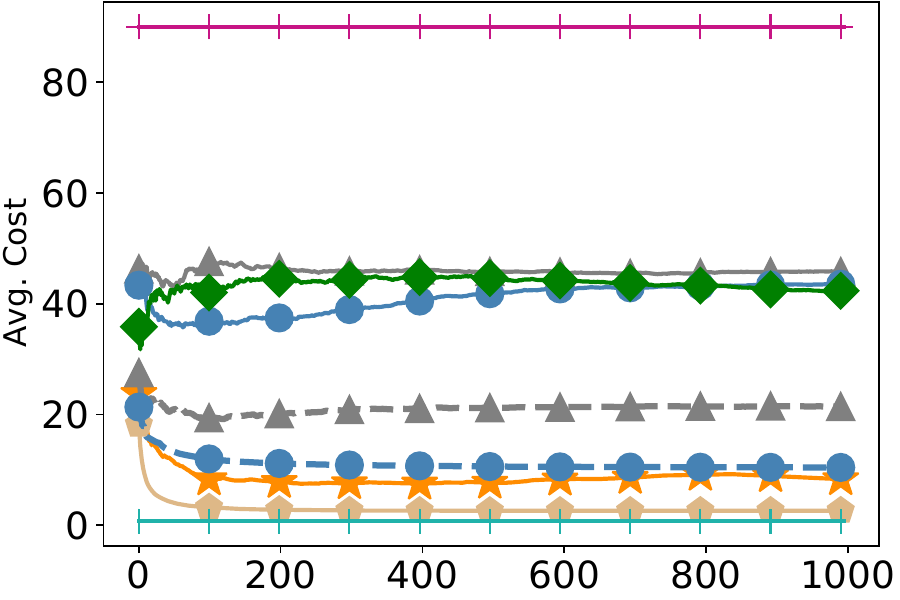} \\
    \vspace{10pt}
    \includegraphics[width=0.40\textwidth]{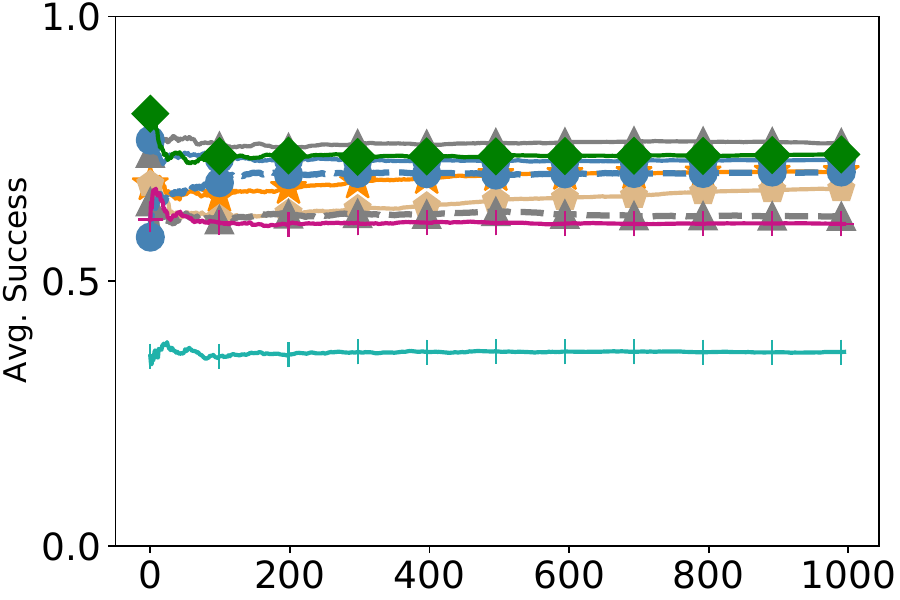}
    \caption{Code translation on the HumanEval-X benchmark~(Task 3): Gemini-2.5-Flash-preview, Deepseek-Chat, Qwen-Plus, GPT-4o, and Claude-Opus-4. Results are averaged over \num{20} trials.}
    \label{fig:humaneval-x-full}
\end{figure}

\begin{figure}[!ht]
    \centering
    \includegraphics[width=0.40\textwidth]{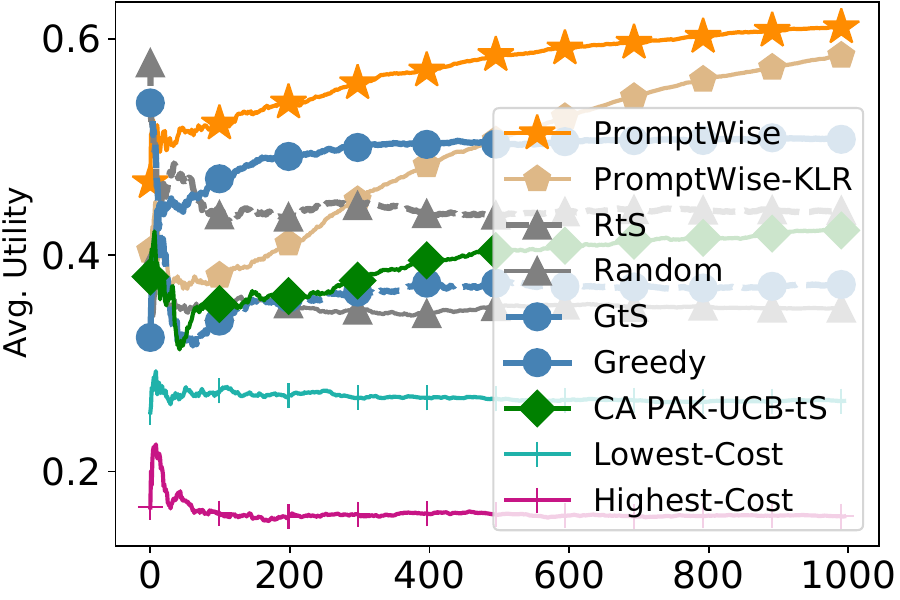}
    \hfill
    \includegraphics[width=0.40\textwidth]{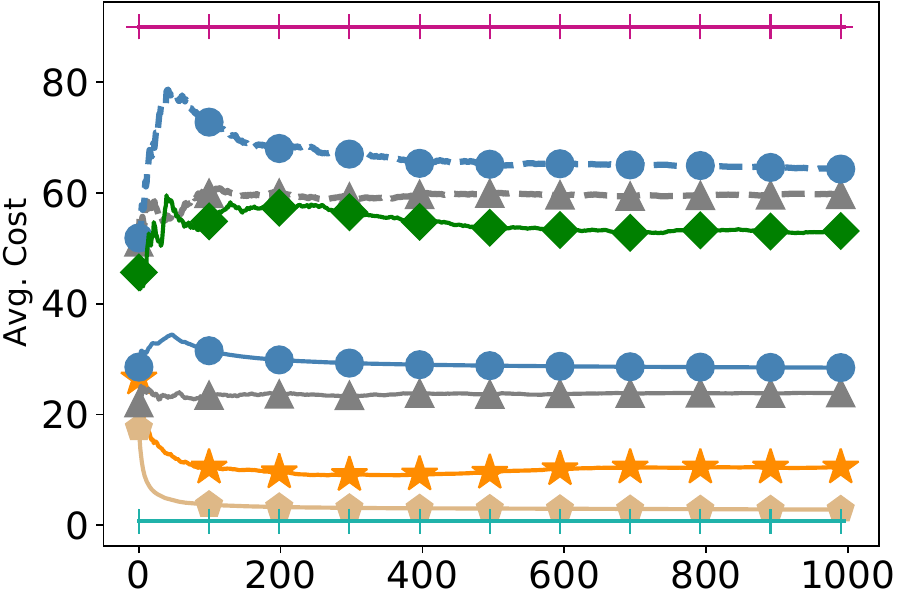} \\
    \vspace{10pt}
    \includegraphics[width=0.40\textwidth]{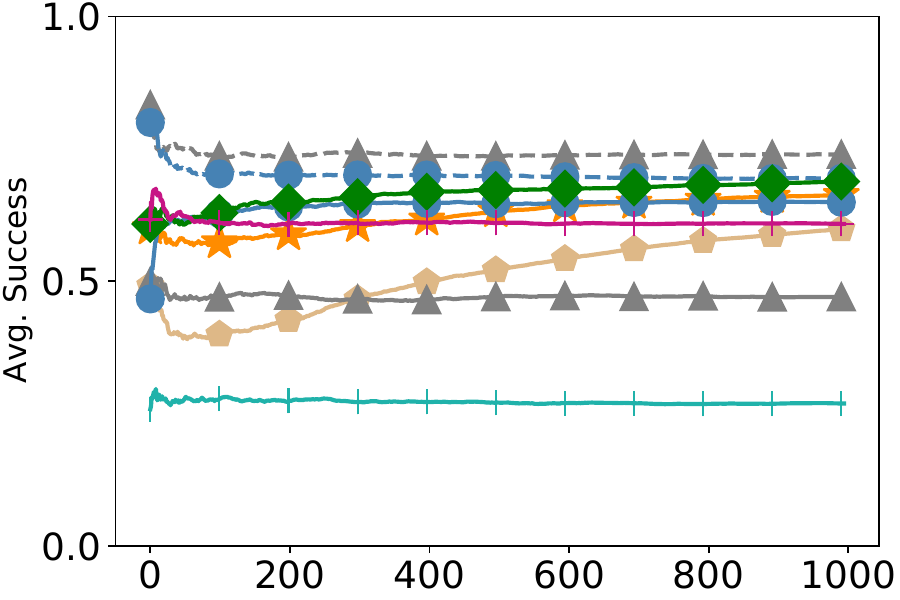}
    \caption{Code translation on the HumanEval-X benchmark~(Task 3): GPT-4o-mini, Deepseek-Chat, Claude-3.5-Haiku, and Claude-Opus-4. Results are averaged over \num{20} trials.}
    \label{fig:humaneval-x-alt}
\end{figure}

\begin{figure}[!ht]
    \centering
    \begin{tabular}{c | c@{\hskip 8pt}c@{\hskip 8pt} c@{\hskip 8pt} c@{\hskip 8pt} c}
        \toprule
        \makecell[c]{Prompt Type \\ (MS-COCO)} & \textbf{Expert 1} & \textbf{Expert 2} & \textbf{Expert 3} & \textbf{Expert 4} & \textbf{Expert 5} \\
        \midrule
        \textbf{Cost} & \num{0.75} & \num{1.37} & \num{1.6} & \num{12.5} & \num{90.0} \\
        \midrule
        carrot & \adjustbox{valign=c}{\includegraphics[width=1.0cm, height=1.0cm,cfbox=green 1pt 1pt]{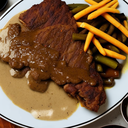}} & \adjustbox{valign=c}{\includegraphics[width=1.0cm, height=1.0cm]{figs/sync-T2I/type3_gen_img.png}} & \adjustbox{valign=c}{\includegraphics[width=1.0cm, height=1.0cm]{figs/sync-T2I/type3_gen_img.png}} & \adjustbox{valign=c}{\includegraphics[width=1.0cm, height=1.0cm]{figs/sync-T2I/type3_gen_img.png}} & \adjustbox{valign=c}{\includegraphics[width=1.0cm, height=1.0cm]{figs/sync-T2I/type3_gen_img.png}} \\
        
        dog & \adjustbox{valign=c}{\includegraphics[width=1.0cm, height=1.0cm]{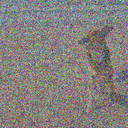}} & \adjustbox{valign=c}{\includegraphics[width=1.0cm, height=1.0cm,cfbox=green 1pt 1pt]{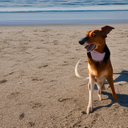}} & \adjustbox{valign=c}{\includegraphics[width=1.0cm, height=1.0cm]{figs/sync-T2I/type1_gen_img.png}} & \adjustbox{valign=c}{\includegraphics[width=1.0cm, height=1.0cm]{figs/sync-T2I/type1_gen_img.png}} & \adjustbox{valign=c}{\includegraphics[width=1.0cm, height=1.0cm]{figs/sync-T2I/type1_gen_img.png}} \\
        
        cake & \adjustbox{valign=c}{\includegraphics[width=1.0cm, height=1.0cm]{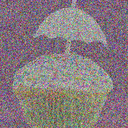}} & \adjustbox{valign=c}{\includegraphics[width=1.0cm, height=1.0cm]{figs/sync-T2I/type4_noisy_gen_img.png}} & \adjustbox{valign=c}{\includegraphics[width=1.0cm, height=1.0cm,cfbox=green 1pt 1pt]{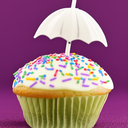}} & \adjustbox{valign=c}{\includegraphics[width=1.0cm, height=1.0cm]{figs/sync-T2I/type4_gen_img.png}} & \adjustbox{valign=c}{\includegraphics[width=1.0cm, height=1.0cm]{figs/sync-T2I/type4_gen_img.png}} \\
        
        car & \adjustbox{valign=c}{\includegraphics[width=1.0cm, height=1.0cm]{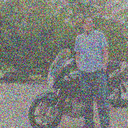}} & \adjustbox{valign=c}{\includegraphics[width=1.0cm, height=1.0cm]{figs/sync-T2I/type2_noisy_gen_img.png}} & \adjustbox{valign=c}{\includegraphics[width=1.0cm, height=1.0cm]{figs/sync-T2I/type2_noisy_gen_img.png}} & \adjustbox{valign=c}{\includegraphics[width=1.0cm, height=1.0cm,cfbox=green 1pt 1pt]{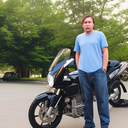}} & \adjustbox{valign=c}{\includegraphics[width=1.0cm, height=1.0cm]{figs/sync-T2I/type2_gen_img.png}} \\
        
        bowl & \adjustbox{valign=c}{\includegraphics[width=1.0cm, height=1.0cm]{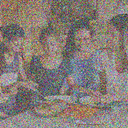}} & \adjustbox{valign=c}{\includegraphics[width=1.0cm, height=1.0cm]{figs/sync-T2I/type5_noisy_gen_img.png}} & \adjustbox{valign=c}{\includegraphics[width=1.0cm, height=1.0cm]{figs/sync-T2I/type5_noisy_gen_img.png}} & \adjustbox{valign=c}{\includegraphics[width=1.0cm, height=1.0cm]{figs/sync-T2I/type5_noisy_gen_img.png}} & \adjustbox{valign=c}{\includegraphics[width=1.0cm, height=1.0cm,cfbox=green 1pt 1pt]{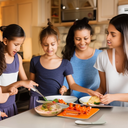}} \\
        \bottomrule
    \end{tabular}
    \caption{Arms in synthetic T2I generation: We set the model costs to align with Tasks 1-3. Prompts are randomly drawn from five categories in the MS-COCO dataset~\citep{lin2015microsoftcococommonobjects}, including 'carrot', 'dog', 'cake', 'car', and 'bowl'. Each row and column displays the generated images from an arm/expert conditioned to the prompt type. The $i$-th expert is constructed such that it outputs a (clean) image from Stable Diffusion v2 for the first $i$ prompt types, while outputting a white-noise-perturbed image with probability \num{0.5} for the rest of the prompt types.}
    \label{tab: Sync-T2I}
\end{figure}

\begin{figure}[!ht]
    \centering
    \includegraphics[width=0.40\textwidth]{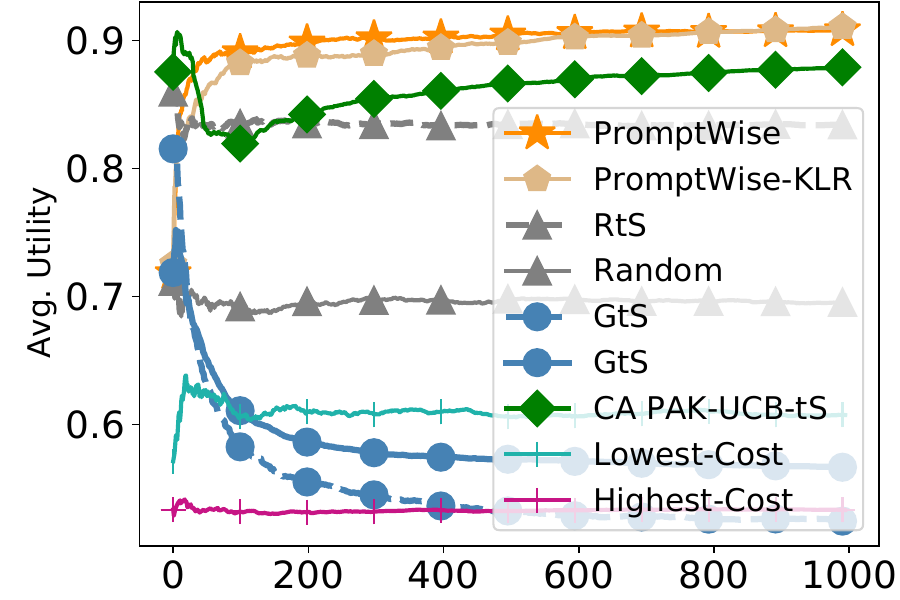}
    \hfill
    \includegraphics[width=0.40\textwidth]{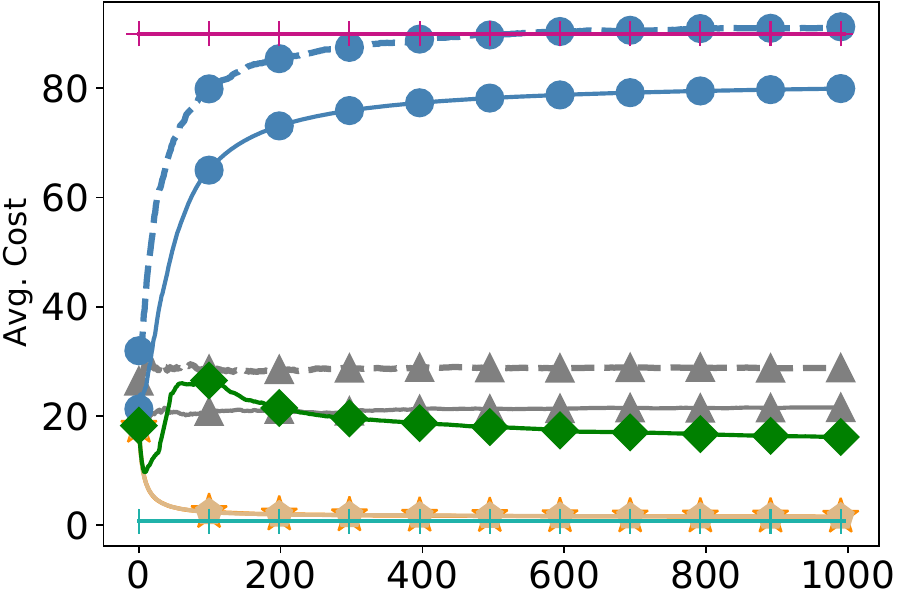} \\
    \vspace{10pt}
    \includegraphics[width=0.40\textwidth]{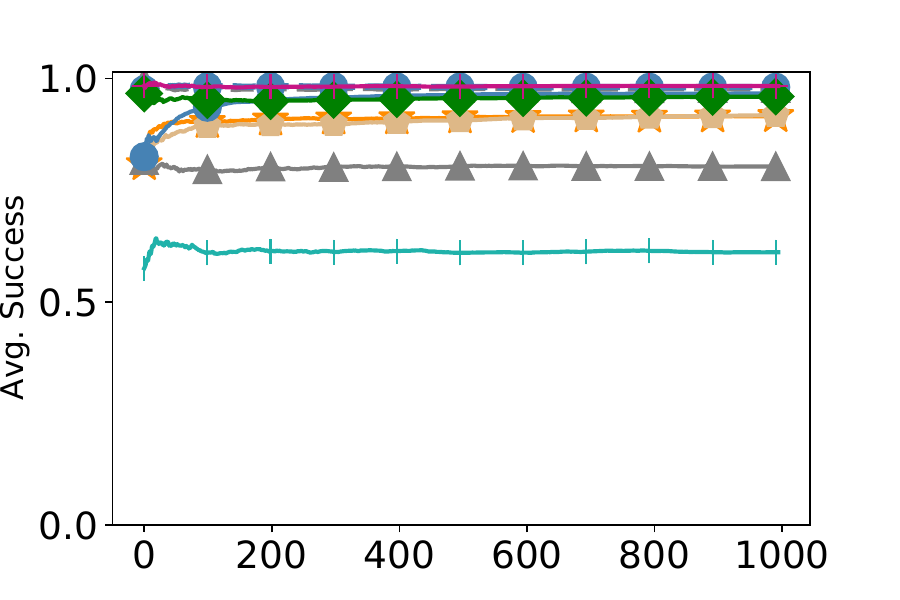}
    \caption{Synthetic Text-to-Image (T2I) Generation (Task 4): The synthetic arms are visualized in Figure~\ref{tab: Sync-T2I}, where $i$-th arm is an expert in generating images for prompts of the first $i$ types. The reward for an generated image $I$ is defined to be $\mathbbm{1}[\text{CLIPScore}(I) \ge 30]$. Results are averaged over 20 trials.}
    \label{fig:sync-T2I-full}
\end{figure}

\subsection{Testing Adaptability of PromptWise.}

We tested {\UCBabbv} on the code translation task under a scenario where new models and prompts are introduced. This scenario is common in practice due to version updates and the release of new models. On the other hand, the interests of a typical user may vary from time to time. In the first setup~(Figure~\ref{fig:new-model-full}), the model is asked to translate between C++ and Java code. Initially, the learner can access two models, i.e., Qwen-Plus and Gemini-2.5-Flash-preview. In the process, GPT-4o, Deepseek-Chat, and Claude-Opus-4 are added to the pool after each \num{200} steps. In the second setup~(Figure~\ref{fig:new-prompt-full}), new translation tasks are introduced in the model selection process. Specifically, the task pool initially contains Java-to-C++ translation problems. Later on, C++-to-Java and Python-to-C++ translation tasks are introduced. Our results show that the proposed {\UCBabbv} can adapt to new models and prompts in the model selection process and significantly outperform the baselines.

\begin{figure}[!ht]
    \centering
    \includegraphics[width=0.40\textwidth]{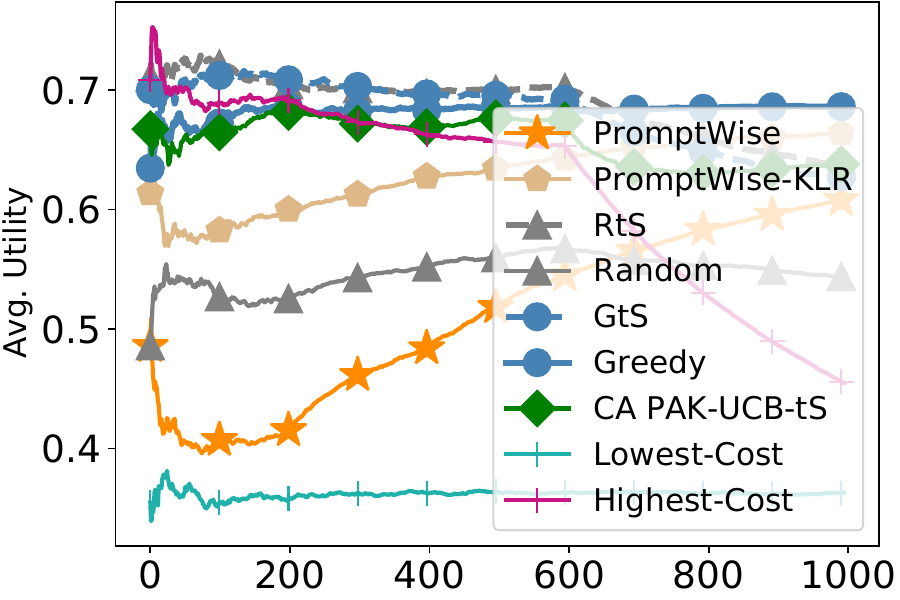}
    \hfill
    \includegraphics[width=0.40\textwidth]{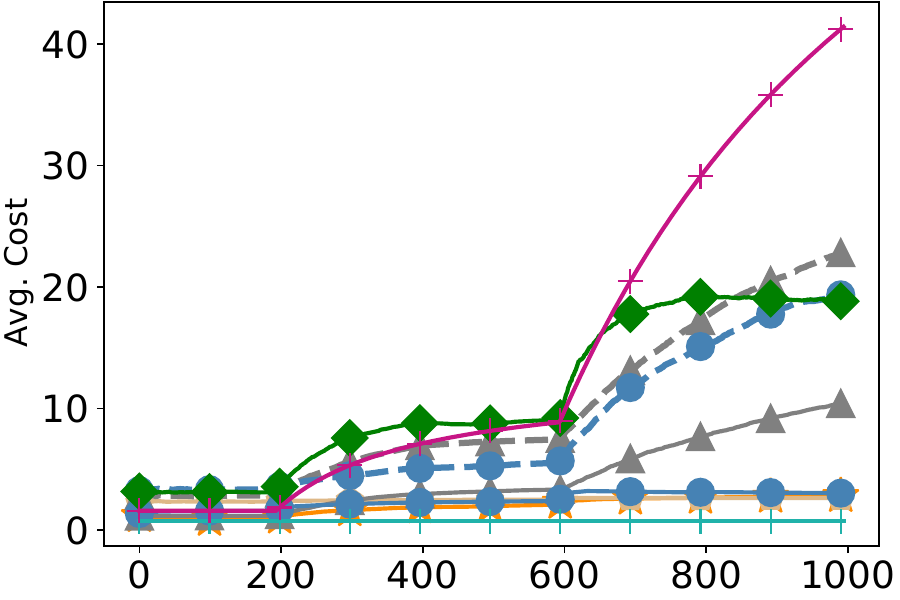} \\
    \vspace{10pt}
    \includegraphics[width=0.40\textwidth]{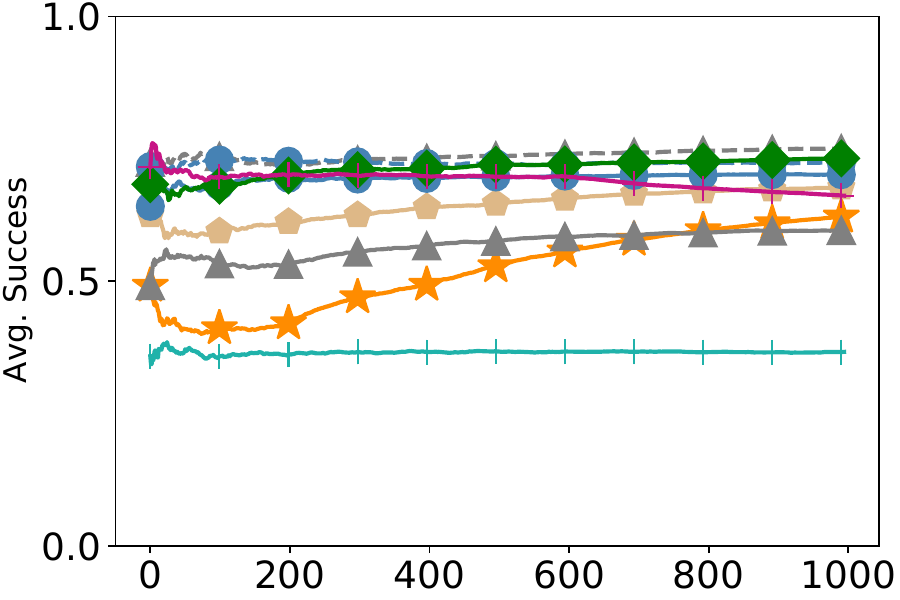}
    \caption{Adaptation to new models on code translation task: Gemini-2.5-Flash-preview, Deepseek-Chat, Qwen-Plus, GPT-4o, and Claude-Opus-4. The model pool contains two LLMs initially, i.e., Qwen-Plus and Gemini-2.5-Flash-preview. In the process, GPT-4o, Deepseek-Chat, and Claude-Opus-4 are added to the pool after each \num{200} steps. Results are averaged over 20 trials.}
    \label{fig:new-model-full}
\end{figure}

\begin{figure}[!ht]
    \centering
    \includegraphics[width=0.40\textwidth]{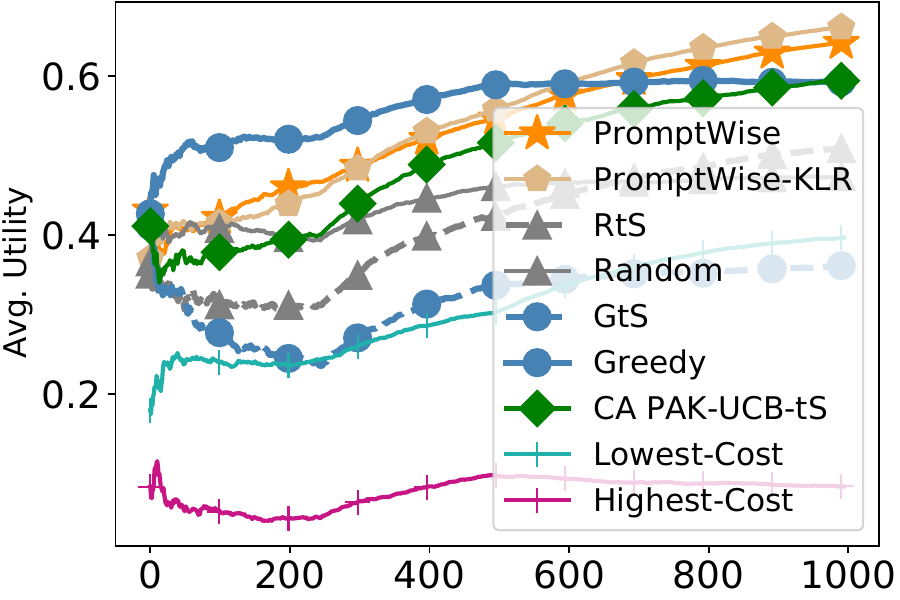}
    \hfill
    \includegraphics[width=0.40\textwidth]{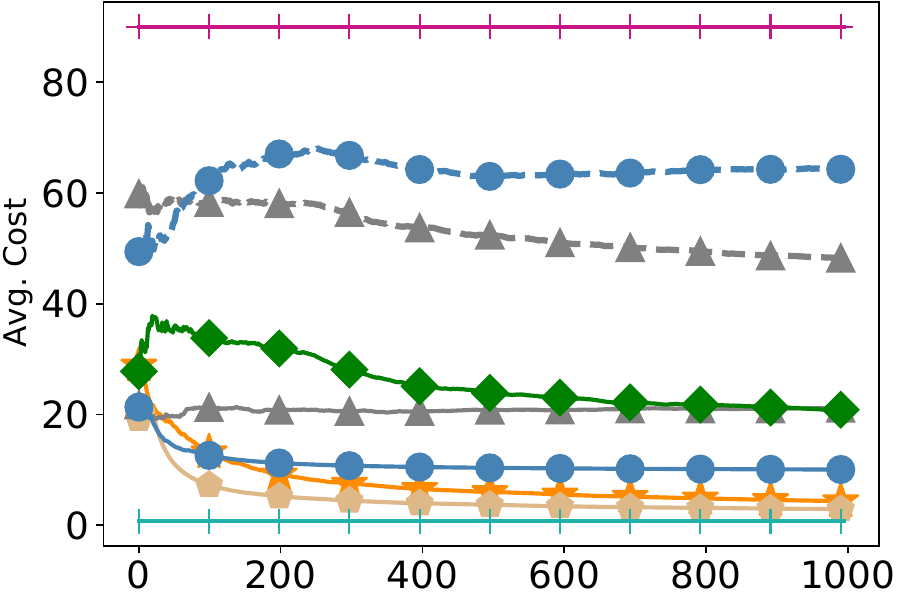} \\
    \vspace{10pt}
    \includegraphics[width=0.40\textwidth]{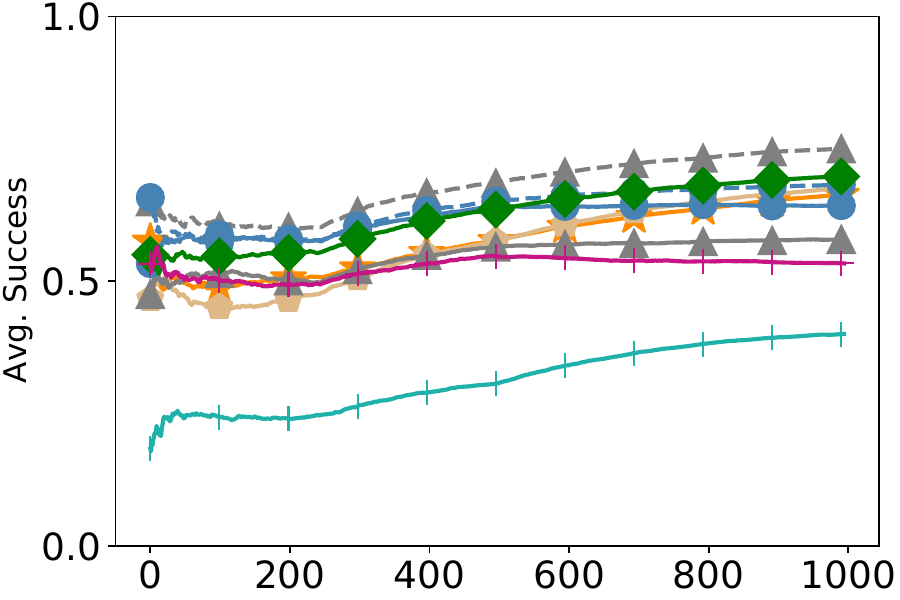}
    \caption{Adaptation to New Prompts: Gemini-2.5-Flash-preview, Deepseek-Chat, Qwen-Plus, GPT-4o, and Claude-Opus-4. Initially, the LLM is asked to solve C++-to-Java translation. Then, Java-to-C++ and Python-to-C++ tasks are introduced after each \num{250} steps. Results are averaged over 20 trials.}
    \label{fig:new-prompt-full}
\end{figure}

\begin{center}
\begin{minipage}{1.0\textwidth}
\lstset{style=mystyle}
\captionof{lstlisting}{Sample prompt for chess puzzle-solving}
\begin{lstlisting}[language=Python, label={lst:chess-puzzle-prompt}]
"""
You are a very strong chess engine. Given the FEN 'r5k1/Rp3p1p/2b2qp1/3pr3/8/4P2P/2PN1PP1/Q3K2R b K - 0 19', it is black's turn to move. Provide the single best *legal* move in Standard Algebraic Notation (SAN). Examples: 'Nf3', 'O-O', 'Rxe5', 'b8=Q'. Do not include commentary, checks ('+'), or checkmates ('#'). Output only the move."
"""
\end{lstlisting}
\end{minipage}
\end{center}

\subsection{The {\UCBabbv}-KLR Algorithm}

We present the {\UCBabbv}-KLR in Algorithm~\ref{alg:ucb-klr} as an extension of {\UCBabbv} to kernel-based prediction functions.

\begin{algorithm}[!t]
\begin{algorithmic}[1]
\caption{{\UCBabbv}-KLR}
\Require step $T$, cost parameter $\lambda > 0$, regression dataset $\gD_a \leftarrow \varnothing$ for $a \in \gA$, $V_a \leftarrow \bm{0}_{d \times d}$, round budget $\tau_{\max}$, tuning parameter $\tau_\textup{exp}$ and $\alpha$, regularization parameter $\beta$, bonus $\gB$
\For{$a \in \gA$}
    \State Collect $\gD_a \leftarrow \gD_a \cup \{ (x_s, r_s) \}_{s = 1}^{\tau_{\exp}}$ by pulling arm $a$ (and moving on to the next context) upon observing contexts $\{x_s\}_{s = 1}^{\tau_{\exp}}$ drawn from the environment.
    \State Calculate the weight $\hatw_a$ by minimizing the regularized negative 
    \begin{equation}
    \label{mle-2-klr}
        \hatw_a \leftarrow \argmin_{w \in \sR^{|\gD_a|}} \left\{ - \sum_{(X,R) \in \gD_a} \left( R \cdot \log \hatq_a(X, w|\gD_a) + (1 - R) \cdot \log (1 - \hatq_a(X, w|\gD_a)) \right) + \beta \| w \|^2_{ \bm{K}_{\gD_a} } \right\}.
    \end{equation}
\EndFor
\For{$t=|\gA|\tau_\textup{exp} + 1, \cdots, T$}
    \State Observe context $x_t$ and compute the estimated success probability for each arm $a \in \gA$:
    \begin{equation}
    \label{opt-hatq-klr}
        \hatq_a \leftarrow \mu\left( \sum_{(X,R) \in \gD_a} \hatw_a^X k(X, x_t) + \alpha \cdot \gB (x_t | \gD_a) \right).
    \end{equation}
    \While{not playing null action $a_0$}
        \If{$\max_{a \in \gA} \{\hatq_a - \lambda\cdot c_a \} \le 0$ \textbf{or} hitting the round budget \textbf{or} receiving a reward of 1}
        \State Play null action $\hata \leftarrow a_0$ and move on to the next step.
        \Else
        \State Play arm $\hata \leftarrow \argmin \frac{c_a}{\hatq_a}$ and receive reward $r$.
        \State Update the regression dataset $\gD_{\hata} \leftarrow \gD_{\hata} \cup \{ (x_t, r) \}$.
        \State Recompute the MLE $\hattheta_a$ and the estimate $\hatq_{\hata}$ by Equations~(\ref{mle-2}) and~(\ref{opt-hatq}), respectively.
        \EndIf
    \EndWhile
\EndFor
\label{alg:ucb-klr}
\end{algorithmic}
\end{algorithm}

\end{document}